\documentclass{article}
\pdfoutput=1

\usepackage{microtype}
\usepackage{graphicx}
\usepackage{subcaption}
\usepackage{booktabs} 

\usepackage{hyperref}

\usepackage{hyperref}
\usepackage{url}

\usepackage[utf8]{inputenc} 
\usepackage[T1]{fontenc}    
\usepackage{hyperref}       
\usepackage{url}            
\usepackage{booktabs}       
\usepackage{amsfonts}       
\usepackage[table]{xcolor}  
\usepackage{caption}        
\usepackage{nicefrac}       
\usepackage{microtype}      
\usepackage{xcolor}         

\usepackage{subcaption} 

\usepackage{hyperref}
\usepackage{multirow} 
\usepackage{algorithmic}

\usepackage{listings}
\usepackage{amsmath}
\usepackage{amssymb}
\usepackage{mathtools}
\usepackage{amsthm}

\usepackage{graphicx}

\usepackage{enumitem}

\usepackage{makecell}

\usepackage{wrapfig} 

\theoremstyle{definition}

\usepackage[preprint]{arxiv}

\usepackage{amsmath}
\usepackage{amssymb}
\usepackage{mathtools}
\usepackage{amsthm}

\usepackage[capitalize,noabbrev]{cleveref}

\theoremstyle{plain}
\newtheorem{theorem}{Theorem}[section]
\newtheorem{proposition}[theorem]{Proposition}

\theoremstyle{definition}
\newtheorem{definition}[theorem]{Definition}
\newtheorem{assumption}[theorem]{Assumption}
\theoremstyle{remark}

\usepackage[utf8]{inputenc} 
\usepackage[T1]{fontenc}    
\usepackage{hyperref}       
\usepackage{url}            
\usepackage{booktabs}       
\usepackage{amsfonts}       
\usepackage{nicefrac}       
\usepackage{microtype}      
\usepackage{xcolor}         

\title{Generated Images Are Easier to Forget: A Machine Unlearning Perspective for Synthetic Image Detection}

\author{%
  \textbf{Jun Nie}$^{1,2}$ \quad \textbf{Yonggang Zhang}$^{3}$ \quad \textbf{Tongliang Liu}$^{4}$ \\
  \textbf{Yiu-ming Cheung}$^{2}$ \quad \textbf{Bo Han}$^{2}$ \quad \textbf{Xinmei Tian}$^{1}$ \\[0.5em]
  $^{1}$University of Science and Technology of China \quad
  $^{2}$Hong Kong Baptist University \\
  $^{3}$The Hong Kong University of Science and Technology \quad
  $^{4}$The University of Sydney \\
}

\begin{document}

\maketitle

\begin{abstract}
  Robust detection of generated images is critical to counter the misuse of generative models. Existing methods primarily depend on learning from human-annotated training datasets, limiting their generalization to unseen distributions. 
In contrast, large-scale vision models (LVMs) pre-trained on web-scale datasets exhibit exceptional generalization power through exposure to diverse distributions, offering a transformative paradigm for this task.
However, our experimental results reveal that LVMs pre-trained on natural-image-dominated data
can effectively capture the features of both natural and generated images, yielding
comparably low losses and thus limited discriminative capacity between them.
This prompts a key question: \emph{When and how do LVMs exhibit different behaviors when capturing features of natural and generated images?} 
This investigation reveals an insight: during unlearning, LVMs exhibit disparate forgetting dynamics with feature degradation for generated images escalating faster than natural ones.
Inspired by the disparate dynamics, we introduce two detection methods: 1) data-free detection, which prunes model parameters to induce unlearning without data access, and 2) data-driven detection, which optimizes LVMs to unlearn knowledge tied to generated images.
Extensive experiments conducted on various benchmarks demonstrate that our unlearning-based approach outperforms conventional detection methods. 
By recasting the detection task as a problem of machine unlearning, our work establishes a new paradigm for generated image detection.
\end{abstract}

\section{Introduction}
\label{sec:intro}

With the rapid advancements in generative models~\citep{DBLP:conf/nips/DhariwalN21, DBLP:conf/cvpr/RombachBLEO22, DBLP:conf/cvpr/KarrasLA19}, AI-generated images have reached a level of quality that often makes them almost indistinguishable from natural images to the human eye. These developments have unlocked unprecedented potential in areas such as content creation, media, and entertainment, driving innovation across industries. However, the ability to generate hyper-realistic images also introduces significant risks~\citep{frank2020leveraging}, especially regarding the potential for misuse in misinformation, privacy invasion, and identity fraud. Consequently, effective and robust detection of AI-generated images has become essential to promote the responsible development and deployment of generative models while protecting users and organizations from malicious use.

Existing methods for detecting AI-generated images focus primarily on learning a boundary between natural and generated images~\citep{DBLP:conf/cvpr/WangW0OE20, DBLP:conf/cvpr/OjhaLL23, DBLP:journals/corr/abs-2312-10461, DBLP:journals/corr/abs-2312-16649, chen2025dual, cai2025towards} to construct a binary classifier. In this context, these methods typically collect labeled natural and generated images to train binary classifiers, aiming to capture and separate features that uniquely characterize each category. Learning from these collected training images, models can identify subtle distinctions between natural and generated content, yielding impressive detection performance under specific conditions.

Despite their success, these methods face the challenge of domain shifts in two critical aspects, which typically degrades the generalization performance. 
First, their performance is inherently constrained by the generative models employed to generate training images, potentially leading to generalization failures when encountering images produced by novel generative models.
Second, the dependency on natural images introduces cross-domain adaptation challenges, as real-world test environments often contain samples distributionally deviating from training images. Thus, these methods are usually paired with carefully designed data augmentation techniques such as JPEG to enhance their generalization performance~\citep{DBLP:conf/cvpr/WangW0OE20}.

\begin{wrapfigure}{l}{0.45\textwidth}
  \centering
  \vspace{-0.3cm}
  \includegraphics[width=0.43\textwidth, trim=0 10 0 0, clip]{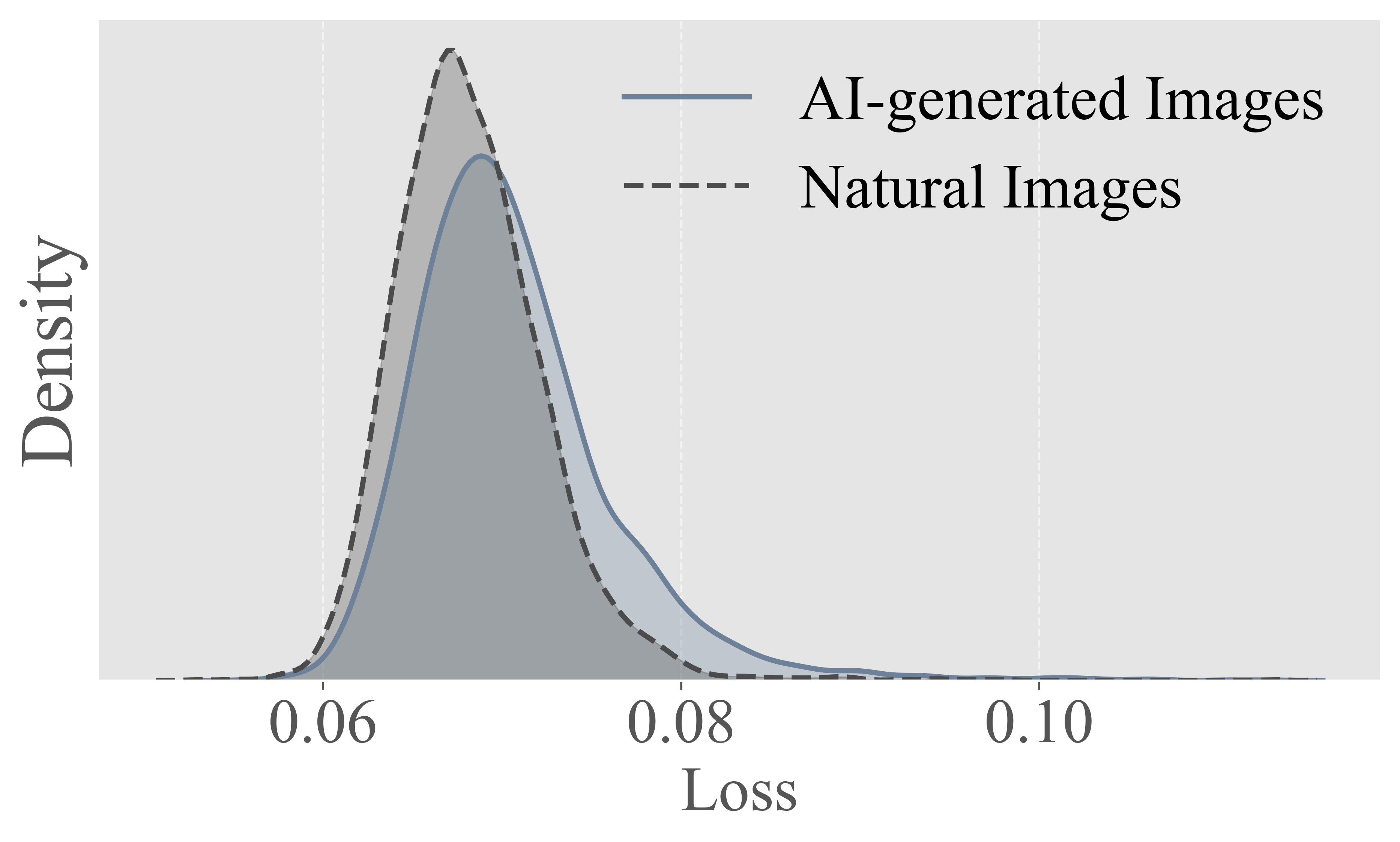}
  \vspace{-0.2cm}
  \caption{LVMs pre-trained on natural images exhibit low losses for both natural and generated images, thus restricting their discriminative
capacity between the types of images.}
  \label{fig:same_loss}
  \vspace{-0.4cm}
\end{wrapfigure}


  

Regarding the ability of generalization, pre-trained vision models (LVMs) have shown great success in various domains thanks to their large-scale training with extensive data~\citep{DBLP:conf/icml/RadfordKHRGASAM21}. Thus, LVMs emerge as promising candidates for promoting the detection of generated images. However, as shown in Figure~\ref{fig:same_loss}, LVMs pre-trained on natural-image-dominated data can effectively capture the features of both natural and generated images to achieve comparably low loss. These experimental results imply that employing LVMs fails to detect generated images correctly, which is consistent with the results shown in Figure~\ref{fig:same_loss}. This prompts a key question: \emph{When and how do LVMs exhibit different behaviors when capturing features of natural and generated images?}

  

\begin{wrapfigure}{l}{0.45\columnwidth}
  \centering
  \vspace{-0.3cm}
  \includegraphics[width=0.46\columnwidth, trim=120 70 370 30, clip]{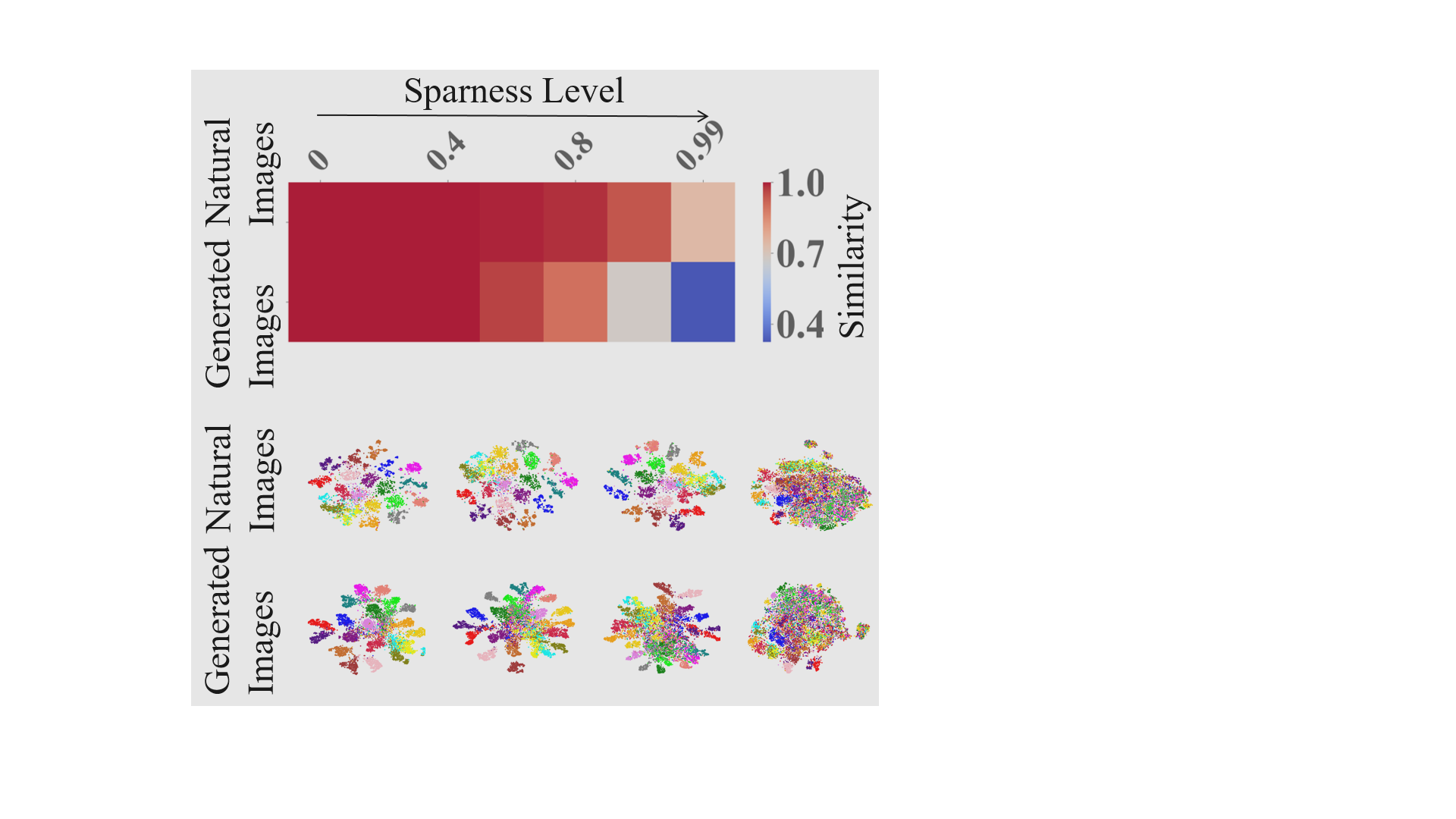}
  \vspace{-0.2cm}
  \caption{Dynamic illustration of unlearning. As the model transitions from learned to unlearned state, the feature extraction ability of the model shows a clear discrepancy between natural and generated images. We compute feature similarity of images on unlearned model and original model, using t-SNE to visualize feature distribution of images on models with different unlearning levels.}
  \label{discrepancy}
  \vspace{-0.4cm}
\end{wrapfigure}

To investigate this fundamental question, we develop a controlled ablation framework to degrade LVM capabilities—a strategy motivated by the inherent difficulty of enhancing LVMs' capabilities, aiming to elicit potential differences in LVMs on two types of images. Specifically, we record and analyze the extent to which the features change as a result of the degraded LVMs' capability. Here, we leverage machine unlearning~\citep{bourtoule2021machine} to degrade LVMs by pruning a specific proportion of parameters with the smallest absolute values. As shown in Figure~\ref{discrepancy}, our experimental results reveal a distinct pattern: as LVMs unlearn more knowledge, the extracted features of natural and generated images change differently. Namely, LVMs exhibit disparate forgetting dynamics with degradation of extracted features for natural and generated images. For original LVMs, different categories of natural and generated images are well distinguished. However, as they unlearn more knowledge, their capability to extract features from the generated images decays significantly faster than from the natural ones. These disparate forgetting dynamics establish a connection between generated image detection and machine unlearning.

Leveraging this insight, we propose an \emph{unlearning-based} approach to detect generated images. Rather than training a binary classifier from scratch to distinguish between natural and generated images, we investigate whether this distinction can be achieved by unlearning knowledge in pre-trained LVMs. Specifically, we propose a simple data-free unlearning method by weight pruning~\citep{DBLP:conf/nips/HanPTD15}, which is an effective approach to compress models by removing unimportant parameters. This method is data-free and training-free, holding potential for generalization, as it does not rely on specific types of generative models and natural images.
Meanwhile, we propose a data-driven method that unlearns generated images to facilitate the separation of natural and generated images. 
Experimental results across multiple benchmarks demonstrate that our unlearning-based approach outperforms state-of-the-art methods. Moreover, we further conduct experiments on images generated by inaccessible generative models, i.e., Sora~\citep{sora}, to verify the robustness against domain shifts of our method. Results shown in Table~\ref{compar_sora} demonstrate that our method consistently and significantly outperforms existing methods.

Our main contributions can be summarized as follows:

\begin{itemize}
    \item 
    Employing LVMs for generated image detection holds promise in addressing the challenge of domain shift, but our experimental results show that original LVMs exhibit similar loss values for both natural and generated images. To elicit potential differences, we develop a controlled ablation experiment and reveal that degraded LVMs display distinct patterns when extracting features for these two types of images.
    \item 
    Inspired by the disparate dynamics, we propose two unlearning-based methods to detect generated images: 1) data-free detection inducing unlearning by parameter pruning and 2) data-driven detection optimizing LVMs by unlearning knowledge tied to generated images. 
    This unlearning-based approach shifts the focus from learning boundaries between natural and generated images to unlearning knowledge in pre-trained models.
    \item Comprehensive experiments validate our method across diverse generated image datasets, demonstrating that our method outperforms existing methods. Moreover, experiments on images generated by inaccessible models verify its robustness against domain shifts.
\end{itemize}

\vspace{-0.3cm}

\section{Preliminaries}
\vspace{-0.3cm}
Given a test image \(\mathbf{x}\), the task of AI-generated image detection is to determine whether \(\mathbf{x}\) originates from the natural image distribution or is generated by a generative model. A common approach frames this as a supervised binary classification problem, utilizing a training set comprising labeled samples from both distributions, formalized as follows.

Let \(X^{0} = \{\mathbf{x}^{0}_1, \dots, \mathbf{x}^{0}_{N^0}\}\) represent a set of \(N^0\) AI-generated images labeled as \(0\), and \(X^{1} = \{\mathbf{x}^{1}_1, \dots, \mathbf{x}^{1}_{N^1}\}\) denote \(N^1\) natural images labeled as \(1\). The objective is to learn a feature extractor \(F(\cdot; \theta_F)\) and a binary classifier \(D(\cdot; \theta_D)\), parameterized by \(\theta_F\) and \(\theta_D\), respectively, by minimizing a classification loss \(\ell(\cdot)\) on the combined dataset:
\vspace{-0.1cm}
\begin{equation}
D, F = \arg\min_{\theta_D, \theta_F} \ \ell\left(D(F(\mathbf{x}; \theta_F); \theta_D), y\right),
\end{equation}
where \(y \in \{0, 1\}\) is the ground-truth label for input \(\mathbf{x}\).

Once trained, the model computes a decision score \(s(\mathbf{x}) = D(F(\mathbf{x}))\) for each test image \(\mathbf{x}\). A hard prediction is obtained by thresholding this score at a fixed value \(\tau\):
\vspace{-0.1cm}
\begin{equation}
\label{decision_function}
\operatorname{pred}(\mathbf{x}) =
\begin{cases}
\text{generated}, & \text{if } s(\mathbf{x}) < \tau, \\
\text{natural}, & \text{otherwise}.
\end{cases}
\end{equation}

The robustness of this framework is fundamentally limited by the empirical coverage of the training data. Effective generalization to unseen distributions requires the learned representation \(F(\mathbf{x})\) captures features invariant to variations in generative models. However, training sets often provide limited coverage of the diverse generative mechanisms encountered in practice. To enhance the robustness of AI-generated image detectors, prior works~\citep{DBLP:conf/icml/ChenZYY24, zhu2023gendet} employ techniques such as data augmentation or adversarial training. Despite these efforts, such methods exhibit limited transferability to samples from unseen generative distributions, highlighting the need for approaches rooted in principled distributional modeling beyond empirical discriminative techniques.

\begin{figure*}[t]
  \centering
  \includegraphics[width=0.68\textwidth]{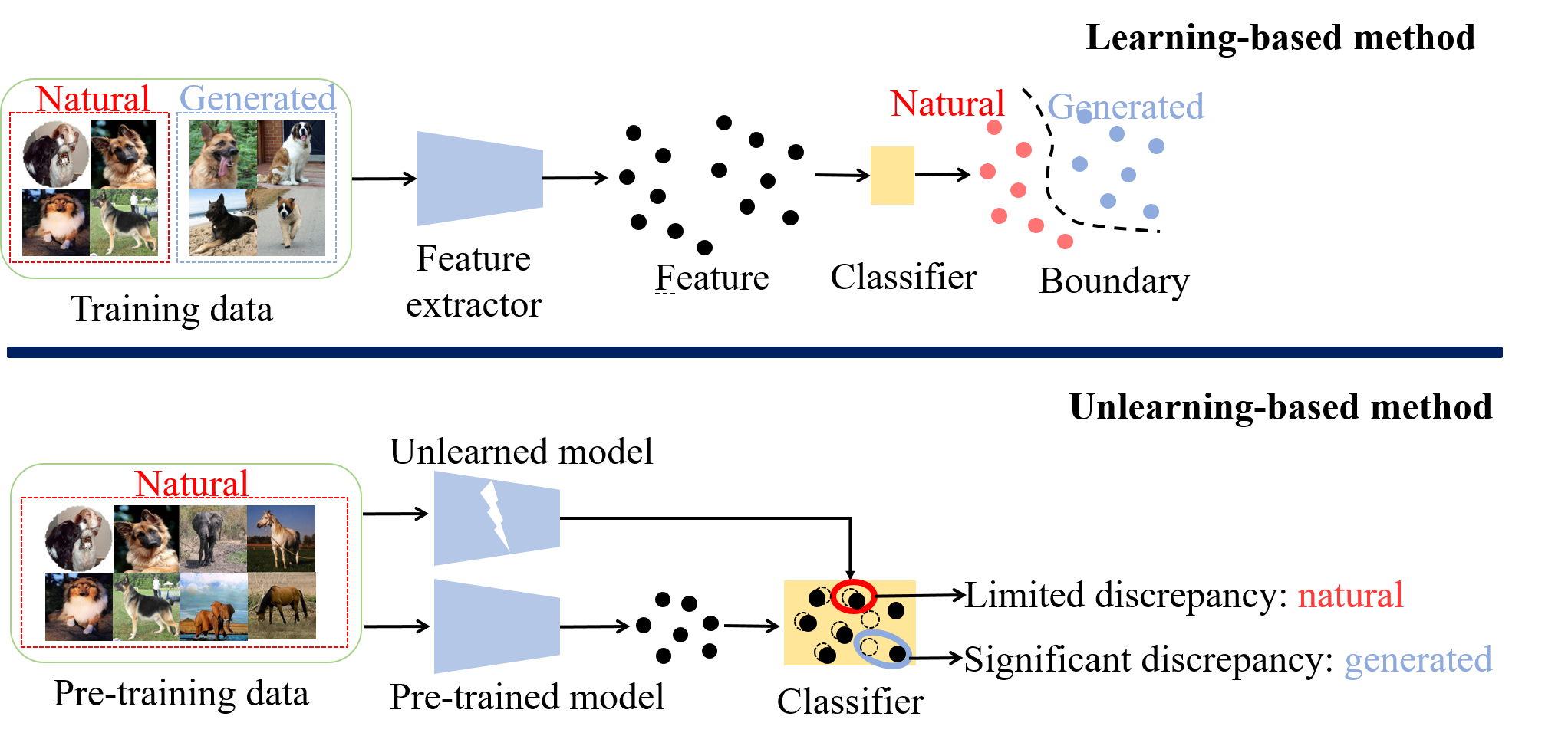}
  \vspace{-0.3cm}
  \caption{Differences and connections between learning-based detection and unlearning-based detection. Learning-based methods aim to introduce a boundary between natural and generated images, leading to the reliance on the collected training data. In contrast, our unlearning-based method leverages dynamic classifiers for detection.}
  \label{pipeline}
  \vspace{-0.5cm}
\end{figure*}

\vspace{-0.2cm}
\section{Methodology}
\vspace{-0.1cm}
\subsection{Motivation}

Current methods for detecting generated images struggle to generalize to unseen generative distributions. A natural approach to improve generalization is to expand the training dataset, leading to the consideration of large-scale pre-trained models such as DINOv2~\citep{DBLP:journals/tmlr/OquabDMVSKFHMEA24}, which offer robust generalization from extensive pre-training. However, as shown in Figure~\ref{fig:same_loss}, these models exhibit comparable low loss on both natural and generated images, reflecting their strong comprehension of both domains, which prevents their direct use for discrimination.

To address this, we propose selectively degrading the model’s ability to interpret generated images, thus inducing differential performance between natural and generated images. We employ machine unlearning~\citep{bourtoule2021machine}, a technique to mitigate the influence of specific data, to adapt pre-trained model to forget generated images while preserving its representation of natural images. A naive unlearning objective function for classification is defined as:

\begin{equation}
\label{unlearnign_obj}
\begin{split}
\mathcal{L}_{\text{unlearn}}(\theta) = & \mathbb{E}_{(x, y) \sim \mathcal{D}_{\text{forget}}} \left[ -\sum_{i} \frac{1}{K} \log f_i(x; \theta) \right] + \lambda \mathbb{E}_{(x, y) \sim \mathcal{D}_{\text{retain}}} \left[ -\sum_{i} y_i \log f_i(x; \theta) \right],
\end{split}
\end{equation}
where, in general, \(\mathcal{D}_{\text{forget}}\) and \(\mathcal{D}_{\text{retain}}\) represent the data distributions to be forgotten and preserved, respectively; in this work, they correspond to generated and natural images.
 



  

\begin{wrapfigure}{r}{0.42\columnwidth}
  \centering
  \vspace{-0.3cm}
  \includegraphics[width=0.40\columnwidth, trim=0 10 0 0, clip]{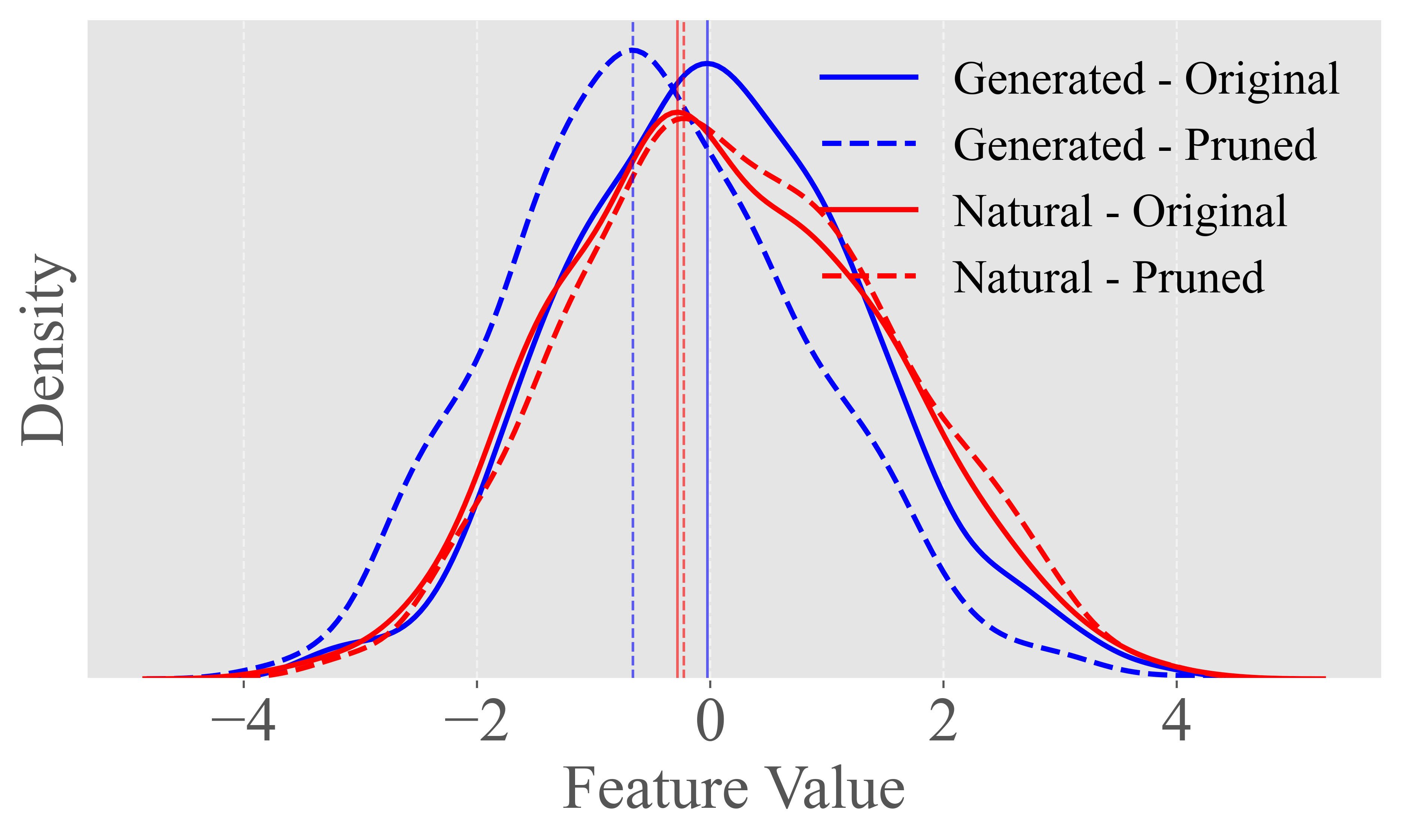}
  \vspace{-0.2cm}
  \caption{The feature shift caused by weight perturbation is more significant in the generated images.}
  \label{fig:feature_shift}
  \vspace{-0.2cm}
\end{wrapfigure}

\subsection{Data-free Unlearning} Eq.~\ref{unlearnign_obj} presents a general machine unlearning framework that requires collecting both natural and generated images to fine-tune the model, incurring additional computational costs. In this study, we investigate strategies to induce selective unlearning in LVMs, prioritizing the retention of natural image knowledge while forgetting generated image knowledge, particularly in scenarios where generated images are unavailable. To do so, we propose a training-free unlearning method by leveraging insights from the pre-training process of large-scale models, achieving effective data-free unlearning. Our idea comes from previous work on the effect of weight pruning on neural networks~\citep{hooker2019compressed}, where the authors found that compression has a greater impact on the long-tail of less frequent instances than on higher frequency instances. For widely used pre-trained vision models such as DINOv2~\citep{DBLP:journals/tmlr/OquabDMVSKFHMEA24} and CLIP~\citep{DBLP:conf/icml/RadfordKHRGASAM21}, the pre-training corpora are typically
dominated by natural images, while synthetic images are either underrepresented or distributionally atypical.
Therefore, generated images tend to behave as relatively rare or distribution-shifted samples with respect to
the learned representation space. Under such a view, weight pruning may disproportionately affect generated
images compared with natural images. As illustrated in Figure~\ref{fig:feature_shift}, pruning the weights of DINOv2 results in a significantly larger feature displacement for generated images than for natural images. This phenomenon is consistently observed in Figures~\ref{discrepancy}.

Based on this, we propose a training-free method for AI-generated image detection, using the feature similarity between a learned model and its unlearned (pruned) version as the criterion:
\begin{equation}
s(x) = \cos(F(x; \theta_F), F(x; \theta_F')),    
 \label{score_function}   
\end{equation}
where \( \cos \) denotes cosine similarity, and \( \theta_F' \) is the unlearned version of \( \theta_F \).


\subsection{A Local Perturbation Analysis of Weight Pruning}
\label{sec:theory_pruning}

While prior studies~\citep{hooker2019compressed} have observed that model compression can affect in-distribution and distribution-shifted samples differently, the mechanism behind such discrepancy remains under-explored.
In this section, we provide a local perturbation analysis to characterize when weight pruning can amplify the representation discrepancy between natural and generated images.

We consider a representation model \( f(x;\theta) \), where \( x \in \mathcal{X} \subset \mathbb{R}^{m} \) is the input image, \( \theta \in \mathbb{R}^{d} \) is the model parameter, and
\begin{equation}
f:\mathcal{X}\times\mathbb{R}^{d}\to\mathbb{R}^{k}
\end{equation}
denotes the output feature embedding.
In our setting, natural images are treated as samples from the reference distribution \( \mathcal{D}_{\mathrm{nat}} \), while generated images are modeled as distribution-shifted samples from \( \mathcal{D}_{\mathrm{gen}} \) with respect to the representation learned from natural-image-dominated pre-training data.

\begin{definition}[Weight Pruning]
\label{def:weight_pruning}
For a model with parameters \( \theta \in \mathbb{R}^{d} \), weight pruning discards weights whose absolute values are below a threshold \( \epsilon>0 \), defined as
\begin{equation}
\theta'_i =
\begin{cases}
\theta_i, & \text{if } |\theta_i| \geq \epsilon, \\
0, & \text{if } |\theta_i| < \epsilon.
\end{cases}
\end{equation}
For notational convenience, we define the pruning direction as
\begin{equation}
    \Delta\theta = \theta-\theta' .
\end{equation}
\end{definition}

\begin{definition}[Pruning-Induced Output Difference]
\label{def:output_difference}
The mean squared output difference induced by weight pruning over distribution \( \mathcal{D} \) is defined as
\begin{equation}
\Delta_{\mathrm{out}}^{\mathcal{D}}
=
\mathbb{E}_{x\sim\mathcal{D}}
\left[
\left\|
f(x;\theta')-f(x;\theta)
\right\|_2^2
\right],
\end{equation}
where \( f(x;\theta) \) and \( f(x;\theta') \) are the representations extracted by the original and pruned models, respectively.
\end{definition}

Since our detector uses the feature discrepancy between the original and pruned models, we analyze the pruning-induced output shift through the Jacobian of the representation with respect to model parameters.

\begin{assumption}[Directional Representation Sensitivity Gap]
\label{assump:representation_gap}
For the pruning direction \( \Delta\theta=\theta-\theta' \), define the directional representation sensitivity over distribution \( \mathcal{D} \) as
\begin{equation}
\mathcal{S}_{\mathcal{D}}(\Delta\theta)
=
\mathbb{E}_{x\sim\mathcal{D}}
\left[
\left\|
J_{\theta}f(x;\theta)\Delta\theta
\right\|_2^2
\right],
\end{equation}
where \( J_{\theta}f(x;\theta) \) denotes the Jacobian of the representation with respect to model parameters.
For the generated-image distributions considered in this work, we assume that generated images exhibit larger directional representation sensitivity than natural images along the pruning direction. Specifically, there exists \( \omega>0 \) such that
\begin{equation}
\label{eq:representation_gap}
\mathcal{S}_{\mathcal{D}_{\mathrm{gen}}}(\Delta\theta)
-
\mathcal{S}_{\mathcal{D}_{\mathrm{nat}}}(\Delta\theta)
\geq
\omega
\|\Delta\theta\|_2^2 .
\end{equation}
\end{assumption}

\begin{proposition}[Output Difference under Directional Sensitivity Gap]
\label{prop:output_difference}
Assume that \( f(x;\theta) \) is twice differentiable with respect to \( \theta \) in a local neighborhood of \( \theta \), and its second-order derivative is locally bounded.
Then there exists a constant \( C_f>0 \) such that, under Assumption~\ref{assump:representation_gap}, the pruning-induced output differences satisfy
\begin{equation}
\label{eq:output_difference_bound}
\Delta_{\mathrm{out}}^{\mathcal{D}_{\mathrm{gen}}}
-
\Delta_{\mathrm{out}}^{\mathcal{D}_{\mathrm{nat}}}
\geq
\omega
\|\Delta\theta\|_2^2
-
C_f
\|\Delta\theta\|_2^3 .
\end{equation}
Consequently, if
\begin{equation}
\label{eq:output_positive_condition}
\omega
>
C_f
\|\Delta\theta\|_2 ,
\end{equation}
then pruning induces a larger output difference for generated images than for natural images:
\begin{equation}
\Delta_{\mathrm{out}}^{\mathcal{D}_{\mathrm{gen}}}
>
\Delta_{\mathrm{out}}^{\mathcal{D}_{\mathrm{nat}}}.
\end{equation}
\end{proposition}

\begin{table*}[h]
\setlength{\tabcolsep}{3pt} 
\caption{AI-generated image detection performance on ImageNet. Values are percentages. \textbf{Bold} numbers are superior results. We compare training methods and training-free methods separately.}
\label{compar_imagenet}
\resizebox{\textwidth}{!}{%
\begin{tabular}{@{}lccccccccccccccccccccccc@{}}
\toprule
                     & \multicolumn{20}{c}{Models}                               & \multicolumn{2}{c}{} \\
 &
  \multicolumn{2}{c}{ADM} &
  \multicolumn{2}{c}{ADMG} &
  \multicolumn{2}{c}{LDM} &
  \multicolumn{2}{c}{DiT} &
  \multicolumn{2}{c}{BigGAN} &
  \multicolumn{2}{c}{GigaGAN} &
  \multicolumn{2}{c}{StyleGAN XL} &
  \multicolumn{2}{c}{RQ-Transformer} &
  \multicolumn{2}{c}{Mask GIT} &
  \multicolumn{2}{c}{\multirow{-2}{*}{Average}} \\ \cmidrule(l){2-3} \cmidrule(l){4-5}\cmidrule(l){6-7}  \cmidrule(l){8-9}\cmidrule(l){10-11}\cmidrule(l){12-13}\cmidrule(l){14-15}\cmidrule(l){16-17}\cmidrule(l){18-19}
\multirow{-3}{*}{Methods}  &
  AUROC &
  AP &
  AUROC&
  AP &
  AUROC&
  AP &
  AUROC&
  AP &
  AUROC&
  AP &
  AUROC&
  AP &
  AUROC&
  AP &
  AUROC&
  AP &
  AUROC&
  AP &
  AUROC&
  AP &\\ \midrule
 &&&&&&&&&\multicolumn{3}{c}{Training-free Methods}&&&&&&&&&&\\
 AEROBLADE &55.61 &54.26 &61.57 &56.58 &62.67 &60.93 &85.88 &87.71 &44.36 &45.66 &47.39 &48.14 &47.28 &48.54 &67.05 &67.69 &48.05 &48.75 &57.87 &57.85 \\
 \rowcolor{pink!20} 
 Data-free Unlearning &\textbf{91.97} &\textbf{90.44} &\textbf{86.82} &\textbf{85.14} &\textbf{87.62} &\textbf{85.91} &\textbf{85.74} &\textbf{83.84} &\textbf{96.37} &\textbf{96.52} &\textbf{94.39} &\textbf{94.23} &\textbf{96.47} &\textbf{96.53} &\textbf{95.19} &\textbf{95.24} &\textbf{95.27} &\textbf{95.17} &\textbf{92.20} &\textbf{91.45}\\
 \midrule

 &&&&&&&&& \multicolumn{3}{c}{Training Methods}&&&&&&&&&&\\
CNNspot  &62.25 &63.13 &63.28 &62.27 &63.16 &64.81 &62.85 &61.16 &85.71 &84.93 &74.85 &71.45 &68.41 &68.67 &61.83 &62.91 &60.98 &61.69 &67.04 &66.78 \\
UnivFD &83.37 &82.95 &79.60 &78.15 &80.35 &79.71 &82.93 &81.72 &93.07 &92.77 &87.45 &84.88 &85.36 &83.15 &85.19 &84.22 &90.82 &90.71 &85.35 &84.25\\
DIRE  &51.82 &50.29 &53.14 &52.96 &52.83 &51.84 &54.67 &55.10 &51.62 &50.83 &50.70 &50.27 &50.95 &51.36 &55.95 &54.83 &52.58 &52.10 &52.70 &52.18 \\
NPR &85.68 &80.86 &84.34 &79.79 &91.98 &86.96 &86.15 &81.26 &89.73 &84.46 &82.21 &78.20 &84.13 &78.73 &80.21 &73.21 &89.61 &84.15 &86.00 &80.84 \\ 
PatchCraft&81.83 &79.65 &70.88 &69.36 &68.47 &65.19 &75.38 &73.29 &99.85 &99.26 &98.55 &97.91 &96.33 &96.25 &91.28 &91.47 &92.56&92.17 &86.13 &84.95\\
FatFormer&91.77 &90.36 &83.58 &83.17 &92.58 &92.06 &86.93 &85.14 &98.76 &98.47 &97.65 &98.02 &97.64 &97.57 &96.55 &95.96 &97.65 &97.27 &93.68 &93.11\\
 DRCT &90.26 &90.07 &85.74 &83.85 &90.24 &89.88 &88.27 &89.06 &95.87 &94.99 &86.89 &86.12 &89.11 &88.39 &92.38 &92.41 &94.44 &94.47 &90.36 &89.92\\
 AIDE &90.87 &90.17 &87.91 &85.52 &93.57 &93.89 &89.87 &88.16 &88.48 &88.12 &97.93 &96.58 &96.59 &95.97 &98.31 &97.86 &99.87 &99.56 &93.71 &92.87\\
 \rowcolor{pink!20} 
 Data-driven Unlearning &\textbf{96.86} &\textbf{96.69} &\textbf{94.92} &\textbf{94.77} &\textbf{98.32} &\textbf{98.50} &\textbf{96.25} &\textbf{96.52} &\textbf{99.96} &\textbf{99.96} &\textbf{99.43} &\textbf{99.54} &\textbf{99.73} &\textbf{99.74} &\textbf{99.26} &\textbf{99.34} &\textbf{99.90} &\textbf{99.91} &\textbf{98.29} &\textbf{98.33} \\
 \bottomrule
\end{tabular}
}
\vspace{-0.2cm}
\end{table*}

\begin{figure*}[h]
\centering
\begin{minipage}{0.48\textwidth}
  \centering
  \vspace{-0.4cm}
  \includegraphics[width=1\linewidth]{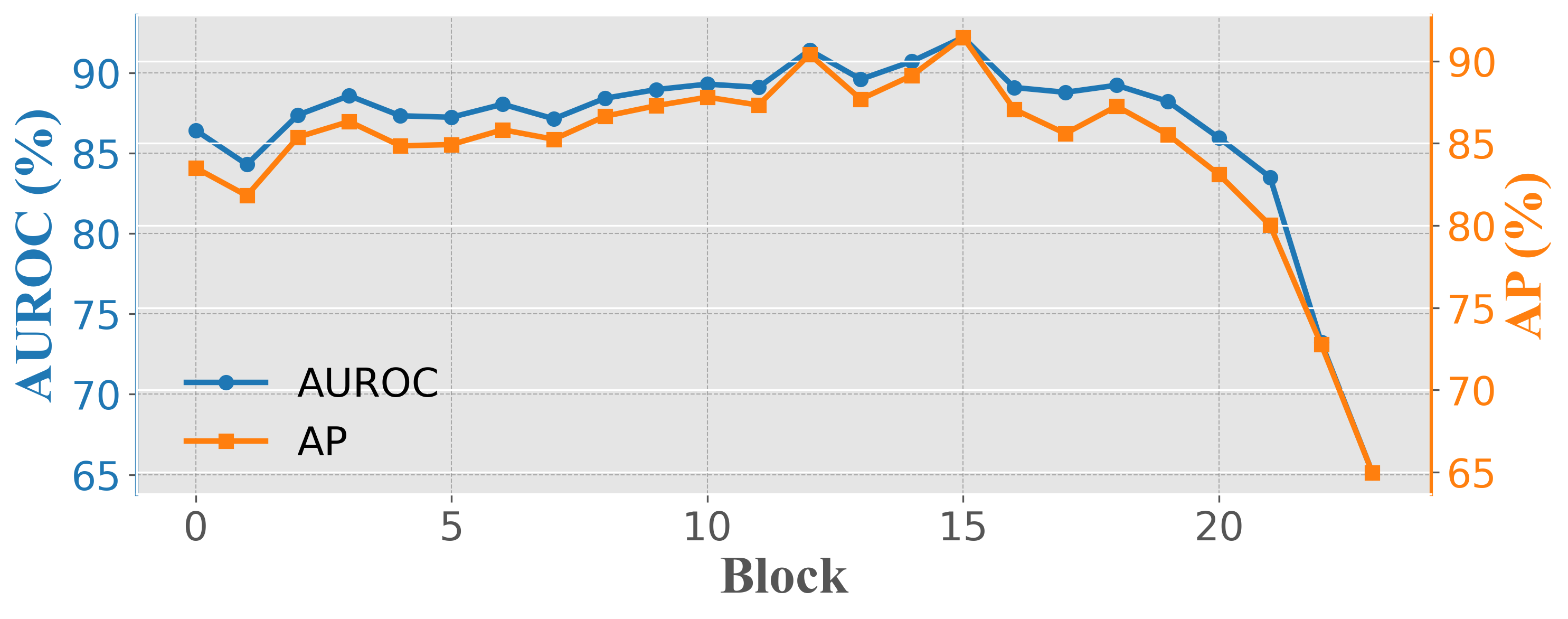}
  \caption{\centering{The effects of pruning blocks.}}
  \vspace{-0.4cm}
  \label{eff_block}
\end{minipage}%
\hfill  
\begin{minipage}{0.48\textwidth}
  \centering
  \includegraphics[width=1\linewidth]{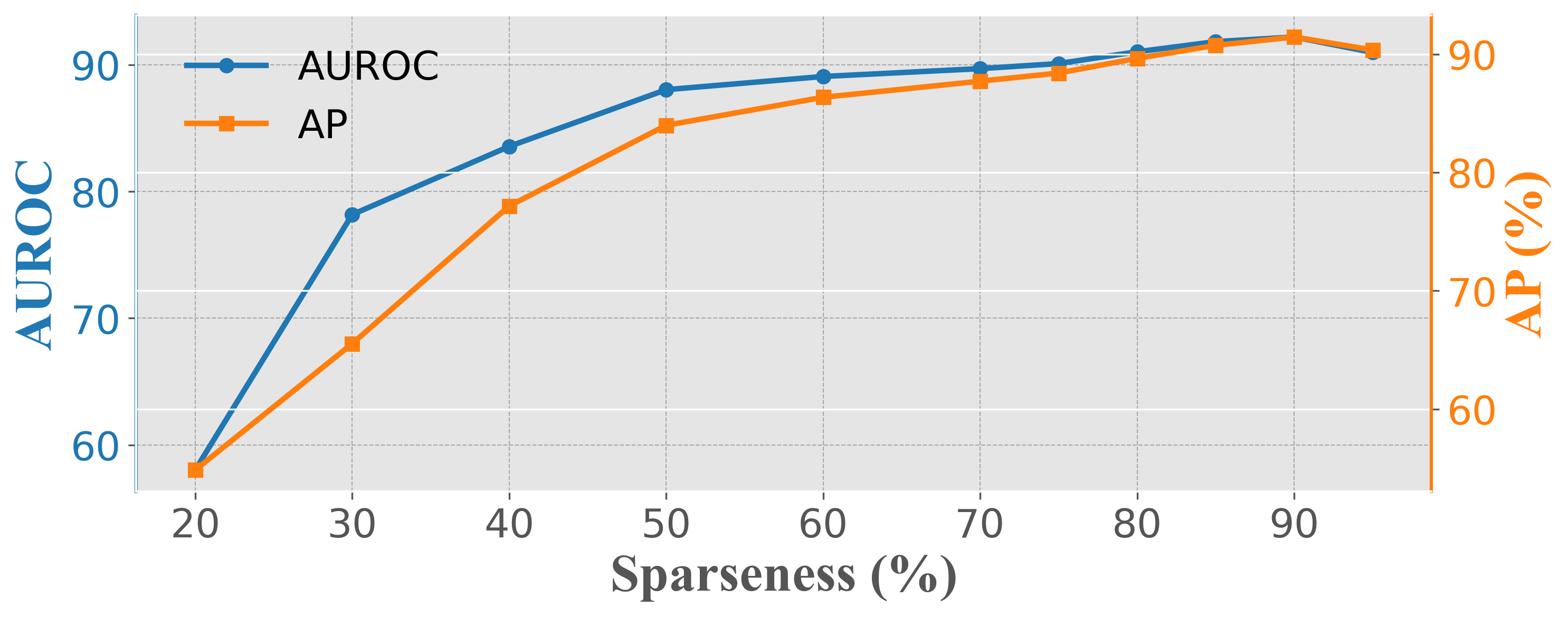}
  \caption{\centering{The effect of pruning ratio.}}
  \label{eff_ratio}
\end{minipage}
\vspace{-0.4cm}
\end{figure*}

\begin{table*}[h]
\centering
\setlength{\tabcolsep}{3pt} 
\caption{Accuracy (\%) of different detectors on Chameleon.}
\label{compar_Chameleon}
\resizebox{1\textwidth}{!}{%
\begin{tabular}{@{}lc>{\columncolor{pink!20}}c|cccccccccccc>{\columncolor{pink!20}}c@{}}
\toprule
 Training Set &AEROBLADE &\multicolumn{1}{c}{\makecell{Data-free\\Unlearning}} &
  CNNSpot &
  FreDect &
  Fusing &
  GramNet &
  LNP &
  UnivFD &
  DIRE 
  &NPR 
  &AIDE
  &DRCT
  &PatchCraft &FatFormer &\multicolumn{1}{c}{\makecell{Data-driven\\Unlearning}}\\
  \midrule
ProGAN &55.29&\textbf{59.17}&56.94 &55.62 &56.98 &58.94 &57.11 &57.22 &58.19 &57.29 &56.45&57.89 &53.76 &55.78&\textbf{60.59}\\
SD v1.4 &55.29&\textbf{59.17}&60.11 &56.86 &57.07 &60.95 &55.63 &55.62 &59.71 &58.13 &61.10 &60.33&56.32 &59.34&\textbf{71.15}\\
All GenImage&55.29 &\textbf{59.17}&60.89 &57.22 &57.09 &59.81 &58.52 &60.42 & 57.83&57.81 &63.89&61.97 &55.70&60.59 &\textbf{72.89}\\
 \bottomrule
\end{tabular}
}
\vspace{-0.2cm}
\end{table*}

\begin{table*}[h]
\centering
\setlength{\tabcolsep}{3pt} 
\caption{AI-generated image detection performance (ACC, \%) on GenImage.}
\label{compar_genimage}
\resizebox{0.75\textwidth}{!}{%
\begin{tabular}{@{}lcccccccccccc@{}}
\toprule
 Methods &
  Midjourney &SD V1.4&
  SD V1.5 &
  ADM &
  GLIDE &
  Wukong &
  VQDM &
  BigGAN &
  Average                   \\
  \midrule
   &&&\multicolumn{3}{c}{Training-free Methods}\\
AEROBLADE &\textbf{80.3} &\textbf{87.5} &\textbf{86.8} &67.2 &\textbf{81.5} &\textbf{83.7} &51.1 &52.5 &73.8\\
 \rowcolor{pink!20}
 Data-free Unlearning &79.9	&80.8	&79.6	&\textbf{77.7}	&79.1	&82.7	&\textbf{87.1}	&\textbf{87.3}	&\textbf{81.8}\\
 \midrule
  &&&\multicolumn{3}{c}{Training Methods}\\
CNNspot &52.8 & 96.3&95.9 &50.1 &39.8 &78.6 &53.4 &46.8 &64.2\\
Spec &52.0 &99.4 &99.2 &49.7 &49.8 &94.8 &55.6 &49.8 &68.8\\
F3Net &50.1 &99.9 &\textbf{99.9} &49.9 &50.0 &\textbf{99.9} &49.9 &49.9 &68.7\\
GramNet &54.2 &99.2 &99.1 &50.3 &54.6 &98.9 &50.8 &51.7 &69.9\\
DIRE &60.2 &99.9 &99.8 &50.9 &55.0 &99.2 &50.1 &50.2 &70.7\\
UnivFD &73.2 &84.2 &84.0 &55.2 &76.9 &75.6 &56.9 &80.3 &73.3\\
PatchCraft &79.0 &89.5 &89.3 &77.3 &78.4 &89.3 &83.7 &72.4 &82.3\\
NPR &81.0 &98.2 &97.9 &76.9 &89.8 &96.9 &84.1 &84.2 &88.6\\
FatFormer &92.7 &\textbf{100.0} &\textbf{99.9} &75.9 &88.0 &99.9 &\textbf{98.8} &55.8 &88.9\\
GenDet &89.6 &96.1  &96.1 &58.0 &78.4 &92.8 &66.5 &75.0 &81.6\\
DRCT &91.5 &95.0 &94.4 &\textbf{79.4} &89.1 &94.6 &90.0 &81.6 &89.4\\
AIDE &79.4 &99.7 &99.8 &78.5 &\textbf{91.8} &98.7 &80.3 &66.9 &86.9\\
SAFE &\textbf{98.3}	&99.9 &99.7 &59.5 &88.5 &99.3 &90.5 &62.1	&87.2\\
 \rowcolor{pink!20}
 Data-driven Unlearning &90.8 &95.6 &95.1 &74.5 &90.0 &94.6 &91.8 &\textbf{85.1} &\textbf{89.7}\\
 \bottomrule
 \vspace{-0.6cm}
\end{tabular}
}
\end{table*}

\begin{table*}[h]
\centering
\begin{minipage}{0.48\textwidth}
\centering
\caption{Comparison with linear classification.}
\label{com_linear}
\resizebox{\textwidth}{!}{%
\begin{tabular}{c|p{36pt}|p{23pt}|p{36pt}|p{23pt}}
\toprule
Dataset & \multicolumn{2}{c|}{ImageNet} & \multicolumn{2}{c}{LSUN-BEDROOM} \\
\midrule
Methods & AUROC & \hspace{7pt}AP & AUROC & \hspace{7pt}AP \\
\midrule 
Linear classification & 87.83 & 86.49 & 84.72 & 83.57 \\
Data-free Unlearning & \textbf{92.20} & \textbf{91.45} & \textbf{91.66} &\textbf{90.58} \\
\bottomrule
\end{tabular}
}
\end{minipage}\hfill
\begin{minipage}{0.5\textwidth}
\centering
\setlength{\tabcolsep}{3pt}
\caption{Effect of unlearning different parameters.}
\label{eff_param}
\resizebox{\textwidth}{!}{%
\begin{tabular}{@{}l|cccccccc@{}}
\midrule 
Metrics & Query & Key & Value & fc1 & fc2 & all \\
\midrule
AUROC & 87.55 & 88.21 & 90.08 & 88.19 & \textbf{92.20} & 88.81 \\
AP & 85.06 & 86.02 & 88.77 & 87.07 & \textbf{91.45} & 87.13 \\
\bottomrule
\end{tabular}
}
\end{minipage}
\vspace{-0.4cm}
\end{table*}

The proof is provided in Appendix~\ref{app:proofs}.
Proposition~\ref{prop:output_difference} formalizes the mechanism behind our unlearning-induced detection signal: pruning amplifies the representation discrepancy whenever generated images are more sensitive to the pruning direction than natural images.
In this case, the feature similarity between the original and pruned models naturally becomes lower for generated images, thereby providing a discriminative score for detection.
This theoretical characterization is consistent with the empirical forgetting dynamics in Figures~\ref{discrepancy} and~\ref{fig:feature_shift}.

\subsection{Data-driven Unlearning}
\vspace{-0.2cm}
While data-free unlearning passively removes generative knowledge through structural degradation, it does not explicitly optimize for forgetting. When generated images are available, we can further introduce a \textit{data-driven unlearning} strategy to guide the unlearning process.

Specifically, for natural images, we encourage the model to retain feature similarity with the original model. For generated images, we enforce a margin-based separation: the model should not produce feature representations too similar to the original model. This is implemented via the following loss:
\begin{equation}
\label{final_loss}
\begin{aligned}
\mathcal{L}(\theta_F') = &\mathbb{E}_{x \in X^1} \left[\mathcal{L}_{\text{CE}}(F(x; \theta_F),\ F(x; \theta_F'))\right] + \mathbb{E}_{x \in X^0} \left[\max\left(0,\ \gamma -  \mathcal{L}_{\text{CE}}(F(x; \theta_F),\ F(x; \theta_F'))\right)\right],
\end{aligned}
\end{equation}

where $\gamma$ controls the separation margin for generated images, and $\mathcal{L}_{\text{CE}}$ is applied to $\ell_2$-normalized representations as a feature-level surrogate. Following the standard cross-entropy implementation, one normalized feature vector is treated as logits and the other as a soft target, encouraging larger representation discrepancy between the original and unlearned models. After obtaining the unlearned model, we compute the scoring function using Eq.~(\ref{score_function}) and make a judgment using Eq.~(\ref{decision_function}).



\vspace{-0.2cm}

\subsection{Relation between learning and unlearning approach}
\vspace{-0.1cm}
In a sense, our unlearning-based method has the same structure as the learning-based method. In our unlearning method, the feature extractor is instantiated as a pre-trained large-scale vision model, while the weighting of classifiers is instantiated as the output features of the unlearned model on the test sample. Our unlearning method has two advantages over the learning methods: (1) we directly use the large-scale vision model as the feature extractor, which is pre-trained on a large number of natural images, instead of retraining a feature extractor on a limited number of samples. This allows to obtain more powerful features; and (2) the classifier of the learning approach is fixed once the training is completed; instead, our unlearning method generates an instance-specific classifier for each test sample, which allows the division of the feature space to be independent of the specific natural and generated samples. We illustrate the differences and connections between our unlearning-based method and learning-based approach in Figure~\ref{pipeline}.

\vspace{-0.2cm}
\section{Experiments}
\vspace{-0.3cm}

\subsection{Setup}
\vspace{-0.2cm}
\textbf{Datasets.} Following previous works~\citep{DBLP:conf/iccv/WangBZWHCL23,yan2024sanity}, we conduct experiments on the following benchmarks: \textbf{ImageNet}~\citep{DBLP:conf/cvpr/DengDSLL009}, \textbf{GenImage}~\citep{DBLP:conf/nips/ZhuCYHLLT0H023}, \textbf{DiffusionForensics}~\citep{DBLP:conf/iccv/WangBZWHCL23}, \textbf{Chameleon}~\citep{yan2024sanity}, \textbf{LSUN-BEDROOM}~\citep{DBLP:journals/corr/YuZSSX15}, and \textbf{DRCT-2M}~\citep{DBLP:conf/icml/ChenZYY24}. Besides these public datasets, we evaluate our method on a proprietary dataset generated using the Sora and OpenSora models.

\noindent \textbf{Implementation Details.} For data-free unlearning, we leverage fully parameterized DINOv2 ViT-L/14 as the learned model. It has 24 transformer blocks, and we obtain a sparse model by pruning the parameters of $90\%$ of the minimum magnitude weights of the fc2 layer of its 16th transformer block, and use this model as the unlearned model. We use $1k$ natural images sampled from ImageNet and generated images generated by ProGAN to select hyperparameters. For data-driven unlearning, we leverage LoRA~\citep{DBLP:conf/iclr/HuSWALWWC22} for parameter-efficient fine-tuning. The LoRA layers are applied on the q\_proj and v\_proj layers of DINOv2. $lora\_r$ and $lora\_\alpha$ are set to 8. The margin $\gamma$ is set to 20. When calculating the classification accuracy, the threshold is determined with a validation set, and threshold sensitivity is analyzed in Table~\ref{Threshold_sensitivity}. More detailed illustration is provided in Appendix~\ref{app:imple_detail}.

\noindent \textbf{Evaluation metrics.} Following previous works~\citep{DBLP:conf/cvpr/OjhaLL23, DBLP:conf/iccv/WangBZWHCL23}, we take the following metrics: (1) the average precision (AP); (2) the area under the receiver operating characteristic curve (AUROC) and (3) the classification accuracy (ACC).

\vspace{-0.2cm}

\noindent \textbf{Baselines.} We take the following works as baselines: CNNspot~\citep{DBLP:conf/cvpr/WangW0OE20}, UnivFD~\citep{DBLP:conf/cvpr/OjhaLL23}, DIRE~\citep{DBLP:conf/iccv/WangBZWHCL23}, NPR~\citep{DBLP:journals/corr/abs-2312-10461}, PatchCraft~\citep{zhong2023rich}, FatFormer~\citep{DBLP:conf/cvpr/LiuTTW0Z24},  DRCT~\citep{DBLP:conf/icml/ChenZYY24}, AIDE~\citep{yan2024sanity} and AEROBLADE~\citep{DBLP:journals/corr/abs-2401-17879}. In addition to the above works, we have also compared our methods on some of benchmarks with the following works: FreDect~\citep{DBLP:conf/icml/FrankESFKH20}, Fusing~\citep{DBLP:conf/icip/JuJKXNL22}, Durall~\citep{DBLP:conf/cvpr/DurallKK20}, LNP~\citep{DBLP:conf/eccv/LiuYBXLG22}, F3Net~\citep{DBLP:conf/eccv/QianYSCS20}, SelfBland~\citep{DBLP:conf/cvpr/ShioharaY22}, GANDetection~\citep{DBLP:conf/icip/MandelliBBT22}, LGrad~\citep{DBLP:conf/cvpr/Tan0WGW23}, Spec~\citep{DBLP:conf/wifs/0022KC19}, GenDet~\citep{zhu2023gendet}, GramNet~\citep{DBLP:conf/cvpr/LiuQT20} and SAFE~\citep{DBLP:conf/kdd/LiCHJHF25}.

\vspace{-0.2cm}
\subsection{Experimental results}
\vspace{-0.1cm}

\noindent \textbf{Comparison with other baselines.} As shown in Table~\ref{compar_imagenet}, \ref{compar_Chameleon}, \ref{compar_genimage}, \ref{com_drct}, \ref{compar_lsun} and  \ref{compar_DiffusionForensics}, we compare our method with other baselines on ImageNet, Chameleon, GenImage, DRCT-2M, LSUN-BEDROOM, and DiffusionForensics, respectively. Results show that our unlearning approach achieves better performance compared with learning-based approach. Notably, even without any generated images to guide the unlearning process, our simple weight pruning-based unlearning method can achieve good results. And performance is further enhanced when generated images are incorporated to steer the unlearning procedure. To further illustrate the effectiveness of our method, we count the image feature similarity on learned and unlearned models for natural images and generated images, respectively. As shown in Figure~\ref{compare_cos}, the similarity of natural images is significantly higher than that of various generated images, and this difference effectively distinguishes natural images from generated images.

\begin{figure*}[t]
  \centering
  \begin{subfigure}{0.23\textwidth}
    \centering
    \includegraphics[width=\textwidth]{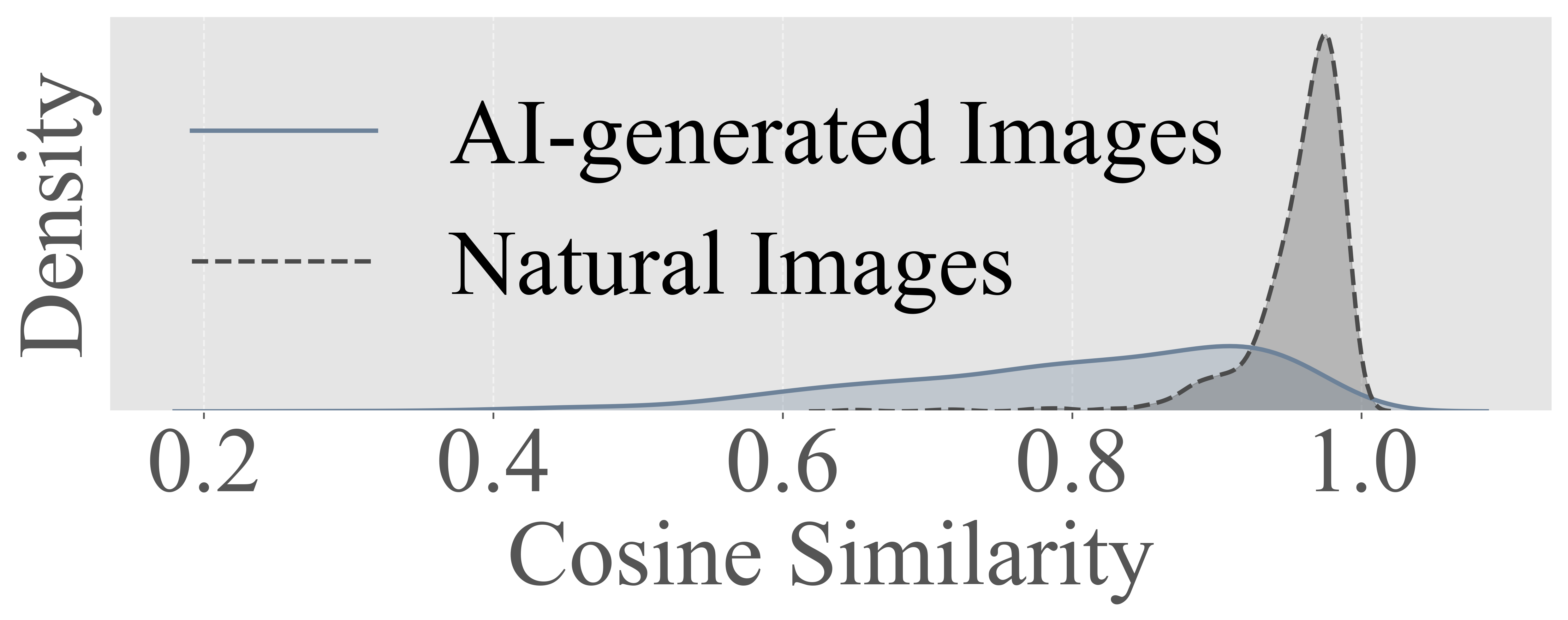}
    \caption{} 
    \label{fig:adm}
  \end{subfigure}
  \hfill 
  \begin{subfigure}{0.23\textwidth}
    \centering
    \includegraphics[width=\textwidth]{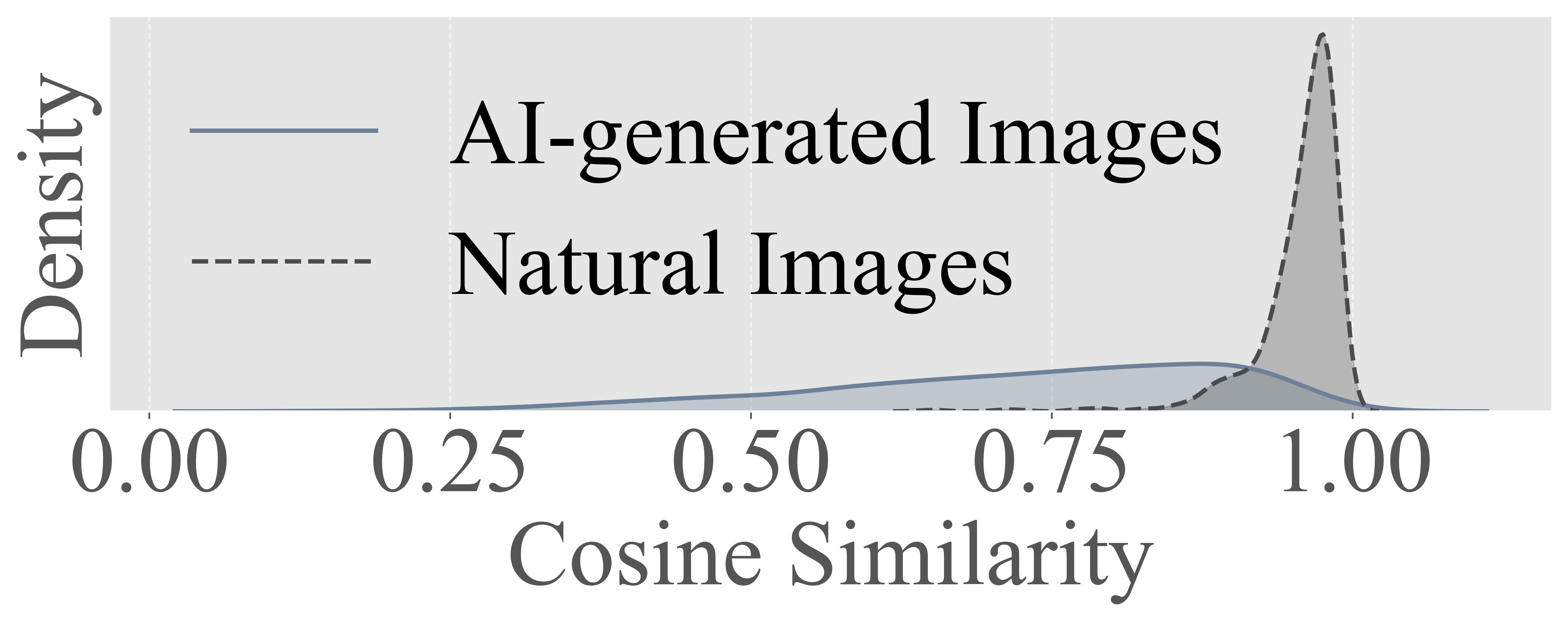}
    \caption{}
    \label{fig:biggan}
  \end{subfigure}
  \hfill
  \begin{subfigure}{0.23\textwidth}
    \centering
    \includegraphics[width=\textwidth]{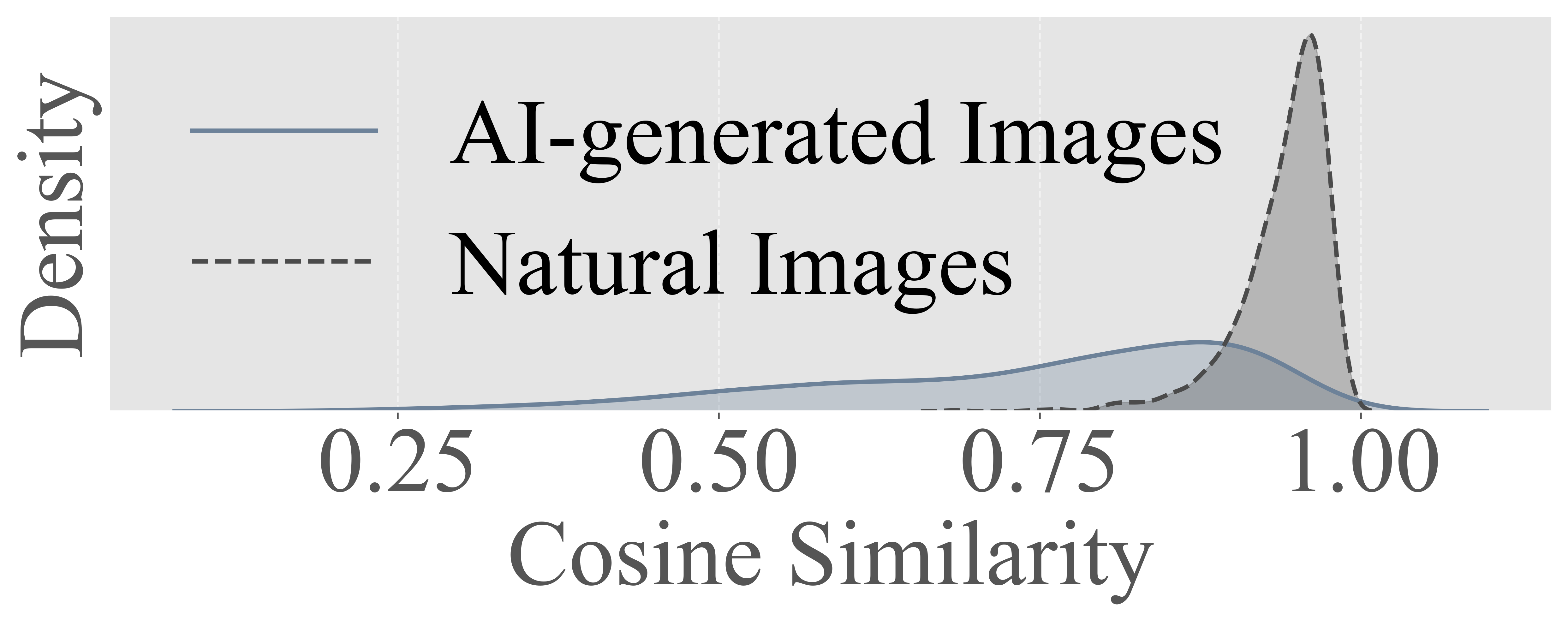}
    \caption{}
    \label{fig:ddpm}
  \end{subfigure}
  \hfill
  \begin{subfigure}{0.23\textwidth}
    \centering
    \includegraphics[width=\textwidth]{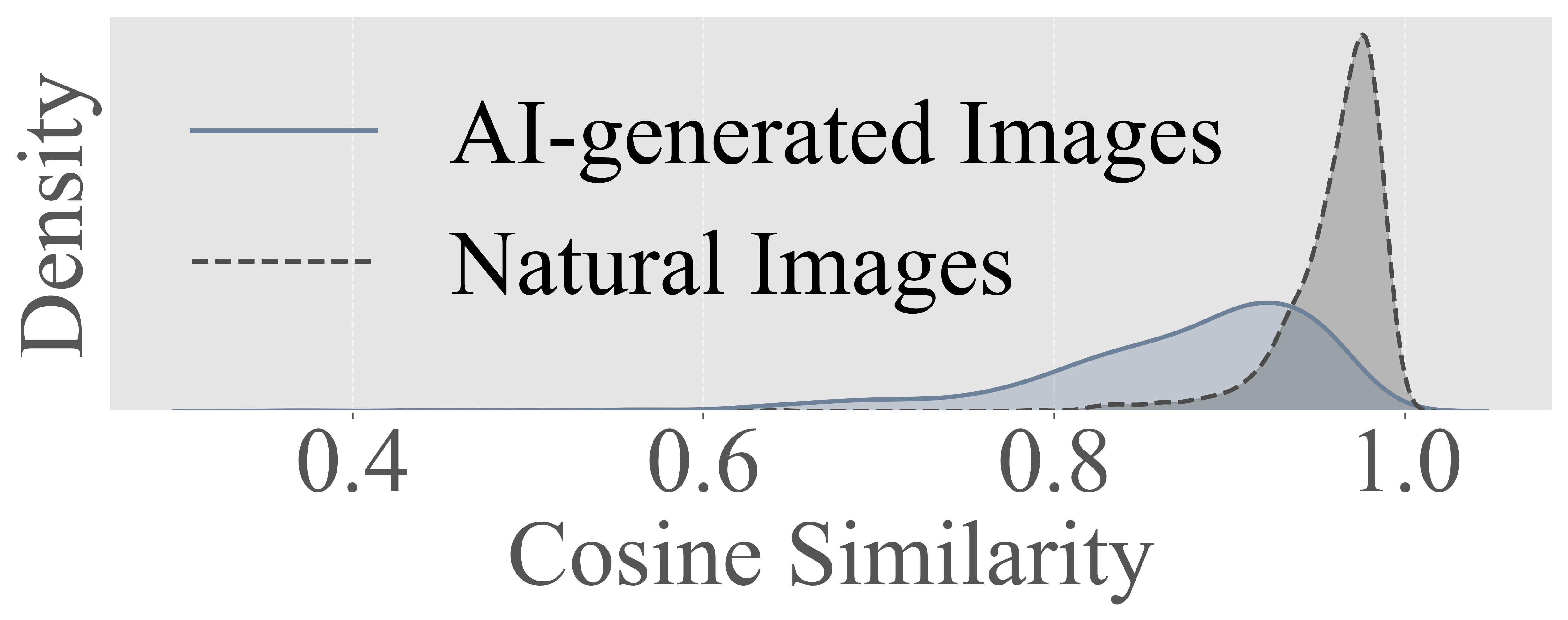}
    \caption{}
    \label{fig:mid}
  \end{subfigure}

  \vspace{-0.2cm} 
  \caption{Comparison of feature similarity on learned and unlearned models. The generated images are from: (a) ADM, (b) BigGAN, (c) DDPM, and (d) Midjourney.}
  \label{compare_cos}
  \vspace{-0.6cm} 
\end{figure*}

\begin{figure*}[t]
  \centering
  \begin{subfigure}{0.32\textwidth}
    \centering
    \includegraphics[width=\textwidth]{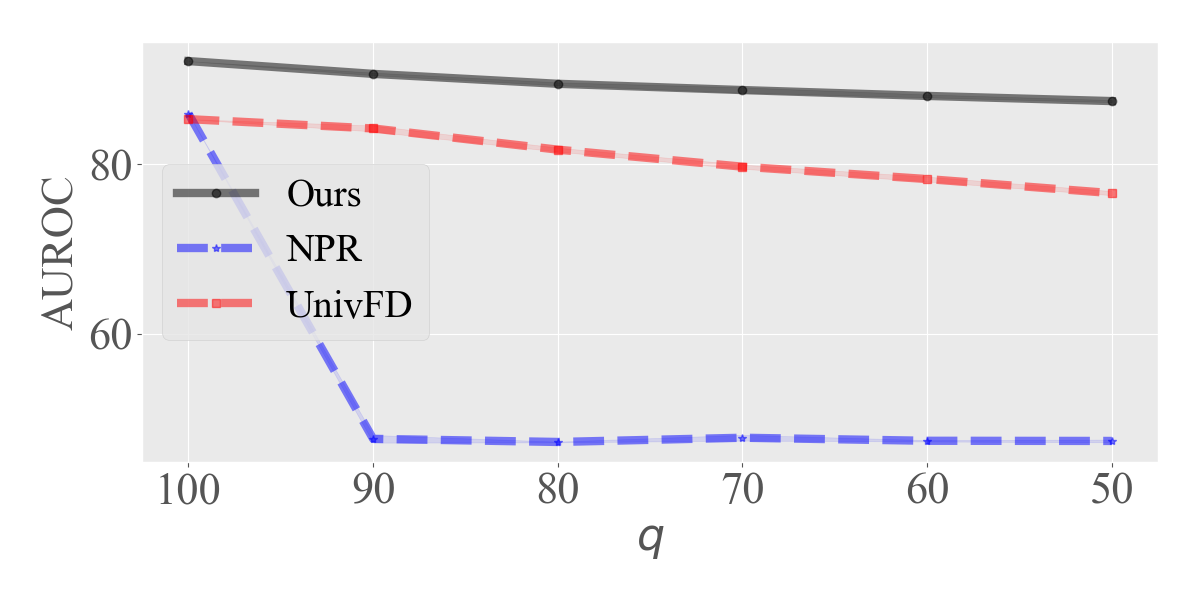}
    \vspace{-0.6cm}
    \caption{} 
    \label{fig:jpeg}
  \end{subfigure}
  \hfill 
  \begin{subfigure}{0.32\textwidth}
    \centering
    \includegraphics[width=\textwidth]{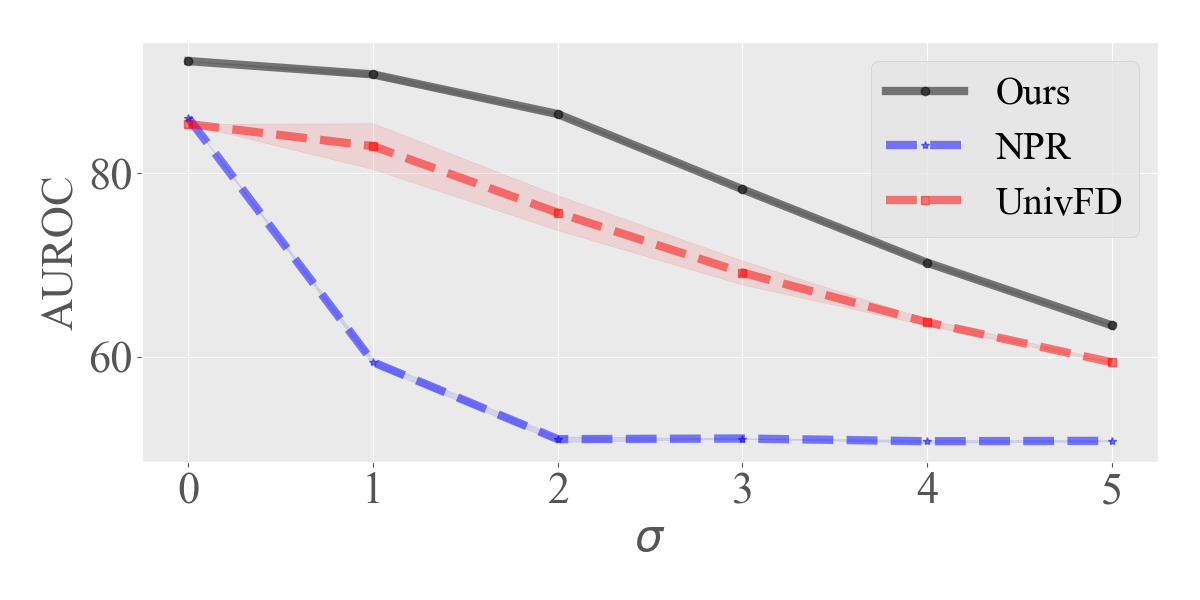}
    \vspace{-0.6cm}
    \caption{} 
    \label{fig:blur}
  \end{subfigure}
  \hfill
  \begin{subfigure}{0.32\textwidth}
    \centering
    \includegraphics[width=\textwidth]{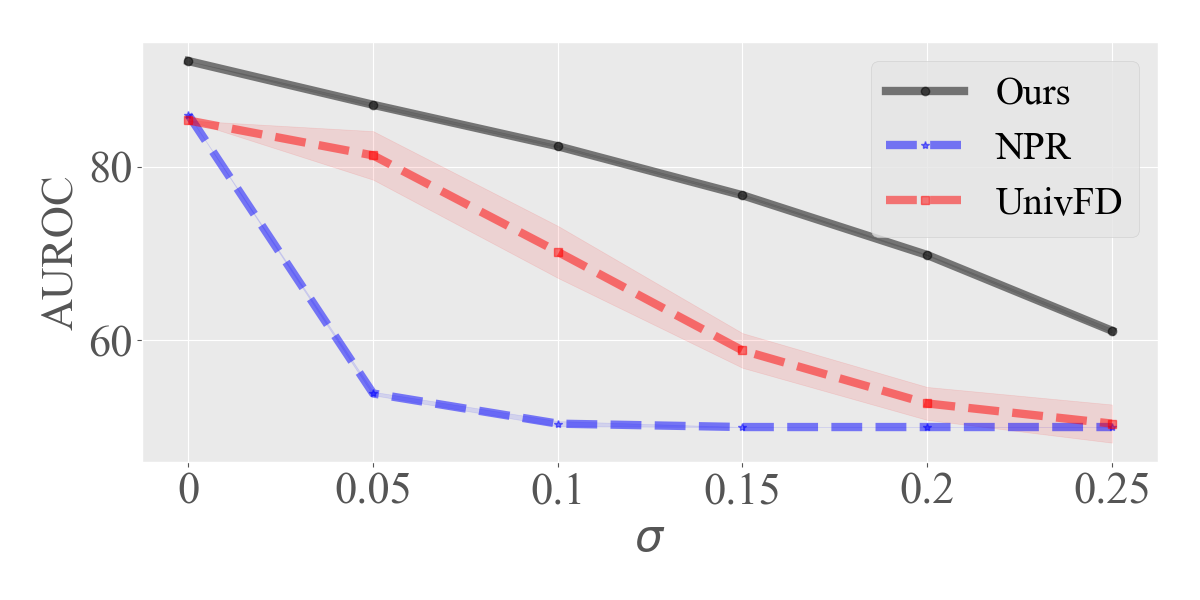}
    \vspace{-0.6cm}
    \caption{} 
    \label{fig:noise}
  \end{subfigure}

  \vspace{-0.2cm} 
  \caption{Robustness to perturbations: (a) JPEG compression; (b) Gaussian blur; (c) Gaussian noise.}
  \label{Perturbations}
  \vspace{-0.5cm} 
\end{figure*}


\vspace{-0.1cm}
\noindent \textbf{Experimental results on Sora.} We further evaluate the performance of our method on video generation models. Specifically, we collect multiple publicly available Sora~\citep{sora} videos and generate multiple videos using Open Sora~\citep{opensora}, and sample images from these videos as AI-generated images and sample natural images from Laion-400M~\citep{schuhmann2021laion} to evaluate the effectiveness of our method. As shown in Table~\ref{compar_sora}, our method achieves the best performance on the images generated by these unknown generative models. These results highlight the effectiveness of the proposed method.

\vspace{-0.1cm}
\subsection{Ablation study}
\vspace{-0.1cm}

\noindent \textbf{The effect of pretrained models.} In our experiments, we use DINOv2 ViT-L/14 as the vision model. To further investigate the effect of vision model, we also evaluate the performance of the unlearning approach on other models. As shown in Table~\ref{eff_model}, the proposed method is able to distinguish between natural and generated images across different models.

\begin{wraptable}{r}{0.45\columnwidth}
\centering
\vspace{-0.4cm}
\caption{The effect of pretrained models.}
\label{eff_model}
\vspace{-0.2cm}
\resizebox{0.4\columnwidth}{!}{%
\begin{tabular}{l|c|c}
\toprule
Model & AUROC & AP \\
\midrule 
DINOv2: ViT-S/14 & 74.01 & 72.86 \\
DINOv2: ViT-B/14 & 85.74 & 83.05 \\
DINOv2: ViT-L/14 & 92.20 & 91.45 \\
DINOv2: ViT-g/14 & 88.12 & 84.73 \\
CLIP: ViT-L/14   & 85.92 & 85.65 \\
CLIP: RN50$\times$64 & 80.03 & 78.32 \\
\bottomrule
\end{tabular}%
}
\vspace{-0.4cm}
\end{wraptable}

\noindent \textbf{Select which block for unlearning?} As shown in Figure~\ref{eff_block}, we investigate the impact of selecting different blocks for unlearning. The results indicate that our method demonstrates robustness across various blocks. However, when unlearning is applied at the top block, performance experiences a significant degradation. This is likely due to the fact that the top block parameters encode higher-level features, and pruning these parameters directly disrupts the feature representation of natural images. As a result, the feature consistency between the learned and unlearned models for natural images is compromised, leading to a notable decline in the performance.

\vspace{-0.1cm}
\noindent \textbf{The effect of pruning ratio.} Figure~\ref{eff_ratio} illustrates the effect of the pruning ratio on the performance of our method. The results show that our method remains robust across a wide range of pruning ratios. Performance degradation occurs only when the pruning ratio is too low, where only limited information is removed and the features extracted by the original and pruned models remain identical.

\begin{wrapfigure}{r}{0.44\columnwidth}
  \centering
  \vspace{-0.3cm}
  \includegraphics[width=0.44\columnwidth]{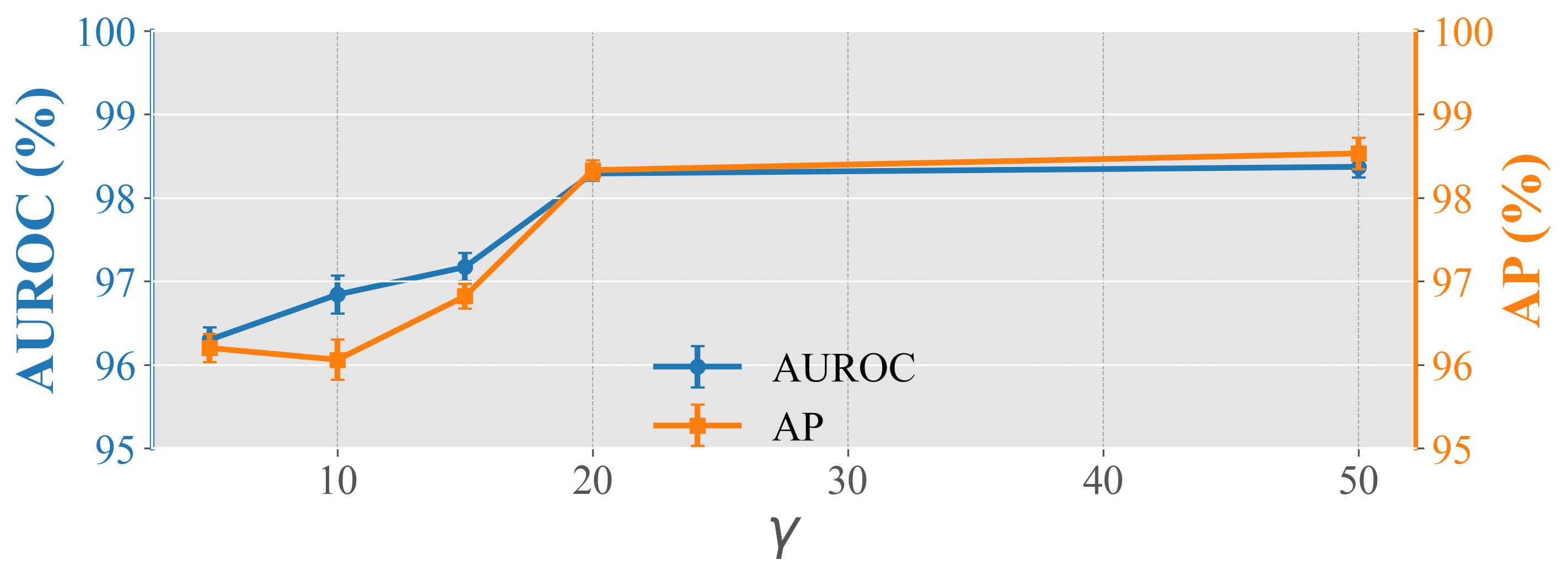}
  \vspace{-0.4cm}
  \caption{The effect of parameter $\gamma$.}
  \label{eff_gamma}
  \vspace{-0.4cm}
\end{wrapfigure}

\noindent \textbf{Robustness to image corruptions.} In addition to the performance on clean images, the robustness of the detector to various image corruptions is also an important metric. In reality, images may be exposed to various perturbations, e.g., when people upload images to social media, images may be compressed, and these operations may affect the performance of the detector. Following previous work~\citep{DBLP:conf/cvpr/WangW0OE20, DBLP:conf/cvpr/OjhaLL23}, we assess the robustness of detectors against three types of perturbations:
JPEG compression (with quality parameter $q$)), Gaussian blur (with standard deviation $\sigma$), and Gaussian
noise (with standard deviation $\sigma$). As shown in Figure~\ref{Perturbations}, our method also achieves the best detection performance under different image perturbations.

\noindent \textbf{The effect of $\gamma$.} Figure~\ref{eff_gamma} illustrates how margin $\gamma$ influences the performance of our data-driven unlearning approach. The results show that our method is robust to $\gamma$.

  



\vspace{-0.1cm}
\noindent \textbf{The effect of parameters.} In DINOv2-L/14, a transformer block consists of two main components: the attention module and the Multi-Layer Perceptron (MLP). Among them, the attention module mainly consists of three parameters: Query, Key and Value. MLP mainly consists of two fully connected layers. In our main experiment, we realize unlearning by pruning the parameters of the second fully connected layer (fc2) in the MLP. In Table~\ref{eff_param}, we further explore the effect of pruning other parameters. The results show that pruning different parameters can obtain good performance.

\noindent \textbf{Comparison with linear classification.} We further compare our method to training a single linear layer for binary classification on top of DINOv2-L/14. We use the training set in UnivFD and use JPEG and Blur as data augmentation methods. As shown in Table~\ref{com_linear}, on the same backbone, our unlearning method also outperforms the learning method.

\noindent \textbf{Ablation experiment on LoRA parameters.} As shown in Table~\ref{eff_lora}, we  conduct ablation experiments on LoRA parameter. The results show that the performance of our unlearning method remains stable under different LoRA parameters, demonstrating the robustness of our method.

\vspace{-0.1cm}
\section{Related Work}
\vspace{-0.2cm}
\label{sec:formatting}

\textbf{AI-generated Image Detection.} In recent years, the detection of AI-generated images has emerged as a critical research area, driven by the rapid advancement of generative models~\citep{DBLP:conf/iclr/BrockDS19, DBLP:conf/nips/HoJA20}. These models can create hyper-realistic images, which raises concerns around issues such as misinformation, privacy violations, and authenticity. To address these concerns, various methods have been proposed to differentiate between natural and AI-generated images. Early methods in this field~\citep{DBLP:conf/iclr/BrockDS19} largely focus on training specialized binary classifiers to distinguish between natural and generated images. For example, CNNspot~\citep{DBLP:conf/cvpr/WangW0OE20} trains a binary classifier using natural and generated images, where JPEG and Blur are used as data augmentation to improve the robustness of the classifier. UniversalFakeDetect~\citep{DBLP:conf/cvpr/OjhaLL23} proposes using CLIP’s representation space~\citep{DBLP:conf/icml/RadfordKHRGASAM21} to train classifiers, which shows superior performance across a wider range of generative architectures. Gendet~\citep{zhu2023gendet} proposes an adversarial teacher-student discrepancy-aware framework, while LaRE$^2$\citep{DBLP:journals/corr/abs-2403-17465} introduces a latent reconstruction error-guided feature refinement approach for detecting images generated by diffusion models. Although these methods have made significant strides, those relying on training still face challenges related to generalization and computational costs. To overcome these limitations, recent studies have shifted their focus toward training-free detection approaches. AEROBLADE\citep{DBLP:journals/corr/abs-2312-10461} takes a training-free approach by assessing reconstruction errors through autoencoder used in Latent Diffusion Model (LDM)~\citep{DBLP:conf/cvpr/RombachBLEO22}, but it is limited to LDM-based generative models. More recently, \citep{nie2025epistemic} leverages the epistemic uncertainty of models pre-trained on natural images to detect generated images, \citep{zhang2025detecting} detects generated images by fitting natural image distributions, and \citep{cai2025towards} explores building a generalizable detector for generated images. In our paper, the proposed unlearning method is also a train-free method and can be generalized to various generative models.

\noindent \textbf{Machine Unlearning.} Machine unlearning has emerged as an important research area due to the need for privacy and security. It focuses on removing the influence of specific data points from a trained model without retraining it from scratch. However, unlearning in deep neural networks is challenging due to their highly non-convex loss functions. \citep{DBLP:conf/cvpr/GolatkarAS20} propose a method for scrubbing the weights clean of information about a particular set of training data. \citep{DBLP:conf/asiaccs/NguyenODCL22} propose a Markov chain Monte Carlo-based machine unlearning algorithm through parameter sampling. These methods balance efficiency and model performance. \citep{DBLP:conf/cvpr/GolatkarARPS21} propose to set a weight subset to zero to effectively remove all the information contained in the non-core data while minimizing the performance loss. \citep{DBLP:conf/nips/JiaLRYLLSL23} utilizes model sparsification via weight pruning to reduce
the gap between exact unlearning and approximate unlearning. In this paper, we perform machine unlearning by sparsification.

\vspace{-0.2cm}
\section{Conclusion}
\vspace{-0.2cm}


In this paper, we propose a novel machine unlearning framework for detecting AI-generated images. Through rigorous analysis of the differential forgetting dynamics in large-scale vision models during unlearning, we establish that feature degradation for generated images outpaces that for natural images. Leveraging this insight, we develop an unlearning-based detection approach that effectively distinguishes generated images from natural ones. Comprehensive evaluations across diverse benchmarks demonstrate superior performance over existing methods.

\nocite{langley00}

\bibliography{example_paper}
\bibliographystyle{apalike}


\appendix

\section{Appendix}

\subsection{Social impacts}
\label{sup:social_impact}

The proposed method for detecting AI-generated images significantly contributes to mitigating societal risks associated with generative model misuse. By enhancing the capability to identify synthetic media, such as deepfakes, this work bolsters efforts to counter disinformation and fosters trust in digital media, particularly in critical domains such as journalism and legal evidence.

\subsection{Limitations}
\label{sup:limitations}
    The proposed method relies on a vision foundation model whose pre-training data are dominated by natural images. As synthetic images become increasingly prevalent in future pre-training corpora, the directional sensitivity gap exploited by our data-free unlearning method may weaken, potentially reducing its effectiveness.


\subsection{Detailed Proofs}
\label{app:proofs}

\subsubsection{Proof of Proposition~\ref{prop:output_difference}}
\label{app:proof_output_difference}

\begin{proof}
For each input \(x\), by the local smoothness of \(f(x;\theta)\) with respect to \(\theta\), we have the first-order expansion
\begin{equation}
\label{eq:feature_taylor}
f(x;\theta')
-
f(x;\theta)
=
-
J_{\theta}f(x;\theta)\Delta\theta
+
e_x ,
\end{equation}
where \(J_{\theta}f(x;\theta)\) is the Jacobian of the representation with respect to model parameters.
Since the second-order derivative of \(f\) with respect to \(\theta\) is locally bounded, there exists a constant \(C_e>0\) such that
\begin{equation}
\label{eq:local_remainder_bound}
\|e_x\|_2
\leq
C_e
\|\Delta\theta\|_2^2 .
\end{equation}
Taking the squared norm on both sides of Eq.~\ref{eq:feature_taylor}, we obtain
\begin{equation}
\label{eq:squared_feature_shift}
\begin{split}
\left\|
f(x;\theta')
-
f(x;\theta)
\right\|_2^2
=
&
\left\|
J_{\theta}f(x;\theta)\Delta\theta
\right\|_2^2
-
2
\left\langle
J_{\theta}f(x;\theta)\Delta\theta,
e_x
\right\rangle
+
\|e_x\|_2^2 .
\end{split}
\end{equation}
Under the local boundedness of \(J_{\theta}f(x;\theta)\), the last two terms can be bounded by \(C_f\|\Delta\theta\|_2^3\) for some constant \(C_f>0\).
Therefore, taking expectation over \(x\sim\mathcal{D}\), we have
\begin{equation}
\label{eq:output_sensitivity_lower}
\Delta_{\mathrm{out}}^{\mathcal{D}}
=
\mathbb{E}_{x\sim\mathcal{D}}
\left[
\left\|
f(x;\theta')
-
f(x;\theta)
\right\|_2^2
\right]
\geq
\mathcal{S}_{\mathcal{D}}(\Delta\theta)
-
C_f
\|\Delta\theta\|_2^3 .
\end{equation}
Similarly, we also have the corresponding upper bound
\begin{equation}
\label{eq:output_sensitivity_upper}
\Delta_{\mathrm{out}}^{\mathcal{D}}
\leq
\mathcal{S}_{\mathcal{D}}(\Delta\theta)
+
C_f
\|\Delta\theta\|_2^3 .
\end{equation}
Applying Eq.~\ref{eq:output_sensitivity_lower} to \( \mathcal{D}_{\mathrm{gen}} \) and Eq.~\ref{eq:output_sensitivity_upper} to \( \mathcal{D}_{\mathrm{nat}} \), and absorbing constants into \(C_f\), we obtain
\begin{equation}
\label{eq:diff_output_sensitivity}
\begin{split}
\Delta_{\mathrm{out}}^{\mathcal{D}_{\mathrm{gen}}}
-
\Delta_{\mathrm{out}}^{\mathcal{D}_{\mathrm{nat}}}
\geq
&
\mathcal{S}_{\mathcal{D}_{\mathrm{gen}}}(\Delta\theta)
-
\mathcal{S}_{\mathcal{D}_{\mathrm{nat}}}(\Delta\theta)
-
C_f
\|\Delta\theta\|_2^3 .
\end{split}
\end{equation}
By Assumption~\ref{assump:representation_gap}, we have
\begin{equation}
\mathcal{S}_{\mathcal{D}_{\mathrm{gen}}}(\Delta\theta)
-
\mathcal{S}_{\mathcal{D}_{\mathrm{nat}}}(\Delta\theta)
\geq
\omega
\|\Delta\theta\|_2^2 .
\end{equation}
Substituting this inequality into Eq.~\ref{eq:diff_output_sensitivity} yields
\begin{equation}
\Delta_{\mathrm{out}}^{\mathcal{D}_{\mathrm{gen}}}
-
\Delta_{\mathrm{out}}^{\mathcal{D}_{\mathrm{nat}}}
\geq
\omega
\|\Delta\theta\|_2^2
-
C_f
\|\Delta\theta\|_2^3 .
\end{equation}
If \( \omega>C_f\|\Delta\theta\|_2 \), the right-hand side is positive. Therefore,
\begin{equation}
\Delta_{\mathrm{out}}^{\mathcal{D}_{\mathrm{gen}}}
>
\Delta_{\mathrm{out}}^{\mathcal{D}_{\mathrm{nat}}}.
\end{equation}
This completes the proof.
\end{proof}

\paragraph{Discussion on the assumption.}
Assumption~\ref{assump:representation_gap} characterizes the directional sensitivity pattern underlying the proposed detection mechanism.
In the considered setting, the pruning direction is induced by small-magnitude weights of a natural-image-dominated pre-trained vision model, and generated images behave as distribution-shifted samples with respect to the learned representation.
This condition is motivated by prior observations that model compression can disproportionately affect long-tail or less frequent samples.
Intuitively, weights that are less critical for preserving natural-image representations may still contribute to fragile representation components of generated images, making generated images more sensitive to pruning.
This sensitivity pattern is empirically supported by the observed forgetting dynamics in Figure~\ref{discrepancy} and the pruning-induced feature shift in Figure~\ref{fig:feature_shift}, where generated images exhibit more pronounced representation changes than natural images.

\subsubsection{Supplementary Risk-Level Analysis}
\label{app:risk_level_analysis}

We further provide a risk-level analogue of the representation-level analysis.
This result is conditional on a directional curvature gap and should be interpreted as a local perturbation analysis of the pruning direction.

Let
\begin{equation}
\mathcal{R}_{\mathcal{D}}(\theta)
=
\mathbb{E}_{(x,y)\sim\mathcal{D}}
\left[
\ell(f(x;\theta),y)
\right]
\end{equation}
be the population risk over distribution \( \mathcal{D} \), and let
\begin{equation}
H_{\mathcal{D}}
=
\nabla_{\theta}^{2}
\mathcal{R}_{\mathcal{D}}(\theta)
\end{equation}
be its Hessian at \( \theta \).
Assume that the risks \( \mathcal{R}_{\mathcal{D}_{\mathrm{nat}}}(\theta) \) and \( \mathcal{R}_{\mathcal{D}_{\mathrm{gen}}}(\theta) \) are twice differentiable in a local neighborhood of \( \theta \), and their Hessians are locally Lipschitz-continuous.
Suppose that there exists \( \rho>0 \) such that
\begin{equation}
\label{eq:risk_curvature_gap_appendix}
\Delta\theta^{\top}
\left(
H_{\mathcal{D}_{\mathrm{gen}}}
-
H_{\mathcal{D}_{\mathrm{nat}}}
\right)
\Delta\theta
\geq
\rho
\|\Delta\theta\|_2^2 .
\end{equation}
Since \( \theta'=\theta-\Delta\theta \), by a second-order Taylor expansion around \( \theta \), there exists a constant \( C_R>0 \) such that
\begin{equation}
\label{eq:risk_taylor_appendix}
\mathcal{R}_{\mathcal{D}}(\theta')
-
\mathcal{R}_{\mathcal{D}}(\theta)
=
-
\nabla_{\theta}\mathcal{R}_{\mathcal{D}}(\theta)^{\top}\Delta\theta
+
\frac{1}{2}
\Delta\theta^{\top}
H_{\mathcal{D}}
\Delta\theta
+
r_{\mathcal{D}},
\end{equation}
where
\begin{equation}
|r_{\mathcal{D}}|
\leq
C_R
\|\Delta\theta\|_2^3 .
\end{equation}
Applying Eq.~\ref{eq:risk_taylor_appendix} to \( \mathcal{D}_{\mathrm{gen}} \) and \( \mathcal{D}_{\mathrm{nat}} \), and then subtracting the two equations, gives
\begin{equation}
\label{eq:risk_difference_appendix}
\begin{split}
&
\Delta\mathrm{Err}_{\mathcal{D}_{\mathrm{gen}}}
-
\Delta\mathrm{Err}_{\mathcal{D}_{\mathrm{nat}}}
\\
=
&
-
\left(
\nabla_{\theta}\mathcal{R}_{\mathcal{D}_{\mathrm{gen}}}(\theta)
-
\nabla_{\theta}\mathcal{R}_{\mathcal{D}_{\mathrm{nat}}}(\theta)
\right)^{\top}
\Delta\theta
\\
&
+
\frac{1}{2}
\Delta\theta^{\top}
\left(
H_{\mathcal{D}_{\mathrm{gen}}}
-
H_{\mathcal{D}_{\mathrm{nat}}}
\right)
\Delta\theta
+
r_{\mathcal{D}_{\mathrm{gen}}}
-
r_{\mathcal{D}_{\mathrm{nat}}}.
\end{split}
\end{equation}
Using Eq.~\ref{eq:risk_curvature_gap_appendix} and absorbing the Taylor remainders into \(C_R\), we obtain
\begin{equation}
\label{eq:risk_bound_appendix}
\begin{split}
\Delta\mathrm{Err}_{\mathcal{D}_{\mathrm{gen}}}
-
\Delta\mathrm{Err}_{\mathcal{D}_{\mathrm{nat}}}
\geq
&
\frac{\rho}{2}
\|\Delta\theta\|_2^2
-
\left|
\left(
\nabla_{\theta}\mathcal{R}_{\mathcal{D}_{\mathrm{gen}}}(\theta)
-
\nabla_{\theta}\mathcal{R}_{\mathcal{D}_{\mathrm{nat}}}(\theta)
\right)^{\top}
\Delta\theta
\right|
\\
&
-
C_R
\|\Delta\theta\|_2^3 .
\end{split}
\end{equation}
Therefore, when the directional curvature gap dominates the first-order residual and the higher-order remainder, the pruning-induced risk increment is larger for generated images than for natural images.

\subsection{Alleviate sensitivity to models through fine-tuning}

As shown in Table~\ref{eff_model}, data-free Unlearning exhibits sensitivity to the choice of feature extractor, as it depends on subtle distinctions in how natural and generated images are represented. For the CLIP model, which is trained with text supervision, extracted image features tend to prioritize semantic information, potentially reducing their suitability for this task. Conversely, the DINOv2:ViT-g/14 model, with its large parameter count, shows limited sensitivity to pruning parameters from a single layer, resulting in minimal impact on the final feature representations. This sensitivity to the backbone can be alleviated through fine-tuning the feature extractor, as demonstrated in the results presented in Table~\ref{driven_unlearning}.

\begin{table}[t] 
\centering
\caption{The effect of LoRA parameters. We report AUROC/AP.}
\label{eff_lora}
\resizebox{0.6\textwidth}{!}{%
\begin{tabular}{c|ccc} 
\toprule
$lora_{\alpha}$ & $r = 4$ & $r = 8$ & $r = 16$ \\
\midrule
8  & 98.72/98.74 & 98.29/98.33 & 98.01/98.20 \\
16 & 98.48/98.53 & 98.71/98.79 & 98.79/98.83 \\
32 & 98.68/98.74 & 98.10/98.21 & 98.61/98.65 \\
\bottomrule
\end{tabular}
}
\vspace{-0.2cm}
\end{table}

\begin{table}[h]
\centering
\caption{Alleviate sensitivity to models through fine-tunings.}
\label{driven_unlearning}
\resizebox{0.7\textwidth}{!}{\begin{tabular}{p{120pt}p{90pt}|p{35pt}|p{20pt}}
\toprule
Method & Model & AUROC & AP \\
\midrule
Data-free Unlearning & DINOv2:ViT-L/14 & 92.20 & 91.45 \\
Data-free Unlearning & DINOv2:ViT-g/14 & 88.12 & 84.73 \\
Data-free Unlearning & CLIP:ViT-L/14   & 85.92 & 85.65 \\
\midrule
Data-driven Unlearning & DINOv2:ViT-L/14 & 98.29 & 98.33 \\
Data-driven Unlearning & DINOv2:ViT-g/14 & 97.96 & 98.15 \\
Data-driven Unlearning & CLIP:ViT-L/14   & 97.63 & 97.48 \\
\bottomrule
\end{tabular}}
\end{table}

\subsection{Performance of data-free unlearning on weaker models}

We conduct additional experiments using weaker pretrained vision models, including MoCo~\citep{DBLP:conf/cvpr/He0WXG20}, SwAV~\citep{DBLP:conf/nips/CaronMMGBJ20}, and DINO~\citep{DBLP:conf/iccv/CaronTMJMBJ21}. As presented in Table~\ref{weak_model}, our Data-free Unlearning exhibit significantly reduced performance on these models. This highlights a promising direction to explore detection with small models. 



\begin{table}[h]
\centering
\begin{minipage}{0.3\textwidth}
\centering
\caption{Performance on weak models.}
\label{weak_model}
\resizebox{\textwidth}{!}{
\begin{tabular}{lcc}
\toprule
Models & AUROC & AP \\
\midrule
MoCo & 72.69 & 70.15 \\
SwAV & 77.85 & 75.64 \\
DINO & 74.79 & 71.88 \\
\bottomrule
\end{tabular}}
\end{minipage}
\hfill
\begin{minipage}{0.53\textwidth}
\centering
\caption{Comparison of training strategies}
\label{com_train_stra}
\resizebox{\textwidth}{!}{
\begin{tabular}{lccc}
\toprule
Methods & AUROC & AP & Training time(min) \\
\midrule
LoRA & 98.29 & 98.33 & 92 \\
Full fine-tuning & 98.37 & 97.72 & 152 \\
Linear probing & 87.83 & 86.49 & 34 \\
\bottomrule
\end{tabular}}
\end{minipage}
\end{table}

\subsection{Comparison of computational efficiency}

As shown in Table~\ref{com_compute_cost}, we conduct a comparative analysis of the training and inference costs of various methods on the ImageNet dataset. For inference cost assessment, we measured the time required to detect 100 images. Experimental results show that our unlearning method also shows advantages in computational cost.

\begin{table}[h]
\centering
\caption{Comparison of computational efficiency.}
\label{com_compute_cost}
\resizebox{0.7\textwidth}{!}{\begin{tabular}{lcc}
\toprule
Methods & training cost (hours) &inference cost (seconds) \\
\midrule
UnivFD &0.8	&1.1\\
NPR	&0.7	&0.7\\
DRCT	&25.4	&1.1\\
AEROBLADE	&0.0	&17.6\\
Data-free Unlearning	&0.0	&2.5\\
Data-driven Unlearning	&1.5	&2.5\\
\bottomrule
\end{tabular}}
\end{table}

\subsection{Comparison of training strategies}

To further validate the effectiveness and efficiency of our unlearning-based approach, we conduct additional comparative experiments using three adaptation strategies on DINOv2: (1) our proposed unlearning via LoRA, (2) full fine-tuning with the same unlearning objective, and (3) linear probing, where a lightweight binary classifier is trained on frozen DINOv2 features using standard cross-entropy loss.
These strategies are evaluated on standard metrics including AUROC and AP, along with training time as a proxy for computational efficiency. The results, summarized in Table~\ref{com_train_stra}, reveal clear trade-offs: while full fine-tuning achieves marginally higher performance in some settings, it incurs substantially higher computational cost. Linear probing offers the fastest training but sacrifices detection accuracy due to its limited expressiveness. In contrast, our LoRA-based unlearning strikes a favorable balance, delivering competitive detection performance with significantly reduced training overhead—making it a practical and efficient choice for our method.

\subsection{Unlearning with different pruning strategy}

In our experiments, we focus on unlearning by pruning the weights with the smallest magnitude. We further explore ablation experiments by randomly pruning the weights, pruning the weights with the largest magnitude, and filling in the smallest magnitude weights with Gaussian noise instead of using 0 after pruning. As shown in Table~\ref{pruning_strategy}, simply pruning the weights with the smallest magnitude for unlearning achieves the best performance.

\subsection{Experimental results on Sora, DRCT-2M, LSUNBEDROOM and DiffusionForensics}

\begin{table*}[!]
\centering
\setlength{\tabcolsep}{3pt} 
\caption{AI-generated image detection performance, measured by AUROC ($\%$) and AP ($\%$), on Sora. }
\label{compar_sora}
\resizebox{1 \textwidth}{!}{%
\begin{tabular}{@{}lcc>{\columncolor{pink!20}}c>{\columncolor{pink!20}}c|cccccccccccccccc>{\columncolor{pink!20}}c>{\columncolor{pink!20}}c@{}}
\toprule
 &
 \multicolumn{2}{c}{AEROBLADE}&
  \multicolumn{2}{c}{\makecell{Data-free \\ Unlearning}}  &
  \multicolumn{2}{c}{CNNspot} &
  \multicolumn{2}{c}{UnivFD} &
  \multicolumn{2}{c}{DIRE} &
  \multicolumn{2}{c}{NPR}&
   \multicolumn{2}{c}{{PatchCraft}}&
   \multicolumn{2}{c}{{FatFormer}} &
  \multicolumn{2}{c}{DRCT}&
   \multicolumn{2}{c}{{AIDE}}&
  \multicolumn{2}{c}{\makecell{Data-driven \\Unlearning}}
  \\ \cmidrule(l){2-3} \cmidrule(l){4-5}\cmidrule(l){6-7}  \cmidrule(l){8-9}\cmidrule(l){10-11}\cmidrule(l){12-13} \cmidrule(l){14-15} \cmidrule(l){16-17} \cmidrule(l){18-19} \cmidrule(l){20-21} \cmidrule(l){22-23}
\multirow{-3}{*}{Models}  &
  AUROC &
  AP &
  AUROC&
  AP &
  AUROC&
  AP &
  AUROC&
  AP &
  AUROC&
  AP &
  AUROC&
  AP &
  AUROC&
  AP &
  AUROC&
  AP &
  AUROC&
  AP &
  AUROC&
  AP &
  AUROC&
  AP 
  \\ \midrule
Sora  &58.00&57.13 &\textbf{90.89} &\textbf{90.18} &52.85 &53.29  &77.06 &80.69 &52.83 &52.16 &51.92 &50.25 &84.39&82.16&89.95&87.64&82.53 &82.28 &91.76&89.39&\textbf{95.89} &\textbf{95.54}\\
Open Sora  &62.37&55.79& \textbf{90.00} &\textbf{89.10} &50.14 &51.38  &67.05 &68.67 &53.66 &52.98 &50.25 &51.84&83.58&81.89&88.76&87.99&81.79 &80.11 &89.47&88.98&\textbf{97.03} &\textbf{95.70} \\
Average  &60.19 &56.46&\textbf{90.45} &\textbf{89.64} &51.50 &52.84  &72.06 &74.68&53.25 &52.57 &51.09 &51.05 &83.99&82.03&89.36&87.82&82.16 &81.20 &90.62 &89.19&\textbf{96.46}&\textbf{95.62}\\
 \bottomrule
\end{tabular}
}
\vspace{-0.3cm}
\end{table*}

\begin{table*}
\centering
\caption{Threshold sensitivity.}
\label{Threshold_sensitivity}
\resizebox{0.7\textwidth}{!}{\begin{tabular}{lcccccccc}
\toprule
Datasets &GLIDE	&Midjourney	&BigGAN	&SD V1.4	&VQDM	&SD V1.5	&Wukong	&ADM \\
\midrule
Threshold &0.93848	&0.94092	&0.94287	&0.93896	&0.91210	&0.94385	&0.93530	&0.94678\\
Acc(\%) &81.69	&81.98	&81.8	&81.94	&79.36	&81.64	&81.79	&81.3\\
\bottomrule
\end{tabular}}
\end{table*}

Table~\ref{compar_sora}, ~\ref{com_drct}, \ref{compar_lsun} and \ref{compar_DiffusionForensics} shows the performance of our unlearning approach on DRCT-2M, LSUNBEDROOM and DiffusionForensics, respectively. The results further demonstrate the effectiveness of our unlearning approach.

\begin{table}[h]
    \centering
    \caption{AI-generated image detection performance (ACC, \%) on DRCT-2M.}
    \label{com_drct}
    \resizebox{1\textwidth}{!}{
    \begin{tabular}{lccccccccccccccccccc}
        \toprule
        \multirow{2}{*}{Method} & \multicolumn{6}{c}{SD Variants} & \multicolumn{2}{c}{Turbo Variants} &\multicolumn{2}{c}{LCM Variants} &\multicolumn{3}{c}{ControlNet Variants} & \multicolumn{3}{c}{DR Variants} &\multirow{2}{*}{Avg.} \\
        \cmidrule(lr){2-7} \cmidrule(lr){8-9}
        \cmidrule(lr){10-11}
        \cmidrule(lr){12-14}
        \cmidrule(lr){15-17}
        
         & LDM & SDv1.4 & SDv1.5 & SDv2 & SDXL & \makecell{SDXL-\\Refiner} & \makecell{SD-\\Turbo} & \makecell{SDXL-\\Turbo} & \makecell{LCM-\\SDv1.5} & \makecell{LCM-\\SDXL} & \makecell{SDv1-\\Ctrl} & \makecell{SDv2-\\Ctrl} & \makecell{SDXL-\\Ctrl} & \makecell{SDv1-\\DR} & \makecell{SDv2-\\DR} & \makecell{SDXL-\\DR} & \\
        \midrule
        CNNSpot  & 99.87 & 99.91 & 99.90 & 97.63 & 66.25 & 86.55 & 86.15 & 72.42 & 98.26 & 61.72 & 97.96 & 85.89 & 82.94 & 60.93 & 51.41 & 50.28 & 81.12 \\
        F3Net  & 99.85 & 99.78 & 99.79 & 88.60 & 55.85 & 87.37 & 63.29 & 63.66 & 97.39 & 54.98 & 97.98 & 72.39 & 81.99 & 65.42 & 50.39 & 50.27 & 71.13 \\
        CLIP/RN50  & 99.00 & 99.99 & 99.96 & 94.61 & 62.08 & 91.43 & 84.40 & 64.40 & 98.97 & 57.43 & 99.74 & 80.69 & 82.03 & 65.83 & 50.67 & 50.47 & 80.05 \\
        GramNet  & 99.40 & 99.01 & 98.84 & 95.30 & 62.63 & 80.68 & 71.19 & 69.32 & 93.05 & 57.02 & 89.97 & 75.55 & 82.68 & 51.23 & 50.01 & 50.08 & 76.62 \\
        De-fake  & 92.1 & 95.53 & 99.51 & 89.65 & 64.02 & 69.24 & 92.00 & 93.93 & 99.13 & 70.89 & 58.98 & 62.34 & 66.66 & 50.12 & 50.16 & 50.00 & 75.52 \\
        Conv-B  & \textbf{99.97} & \textbf{100.0} & \textbf{99.97} & 95.84 & 64.44 & 82.00 & 60.75 & 99.27 & 99.27 & 62.33 & \textbf{99.80} & 83.40 & 73.28 & 61.65 & 51.79 & 50.41 & 79.11 \\
        UniFD  & 98.30 & 96.22 & 96.33 & 93.83 & 91.01 & 93.91 & 86.38 & 85.92 & 90.44 & 89.99 & 90.41 & 81.06 & 89.06 & 51.96 & 51.03 & 50.46 & 83.46 \\
        FatFormer & 96.52 & 95.31 & 93.27 & 91.99 & 92.87 & 91.78 & 88.15 & 87.48 & 92.82 & 91.76 & 90.28 & 86.99 & 88.19 & 65.92 & 60.15 & 55.13 & 85.53 \\
        DIRE  & 54.62 & 75.89 & 76.04 & \textbf{99.87} & 59.90 & 93.08 & 97.55 & 87.29 & 72.53 & 67.85 & 99.69 & 64.40 & 64.40 & 49.96 & 52.48 & 49.92 & 72.55 \\
        DRCT  & 94.45 & 94.35 & 94.24 & 95.05 & 96.41 & 95.38 & 94.81 & 94.48 & 91.66 & 95.54 & 93.86 & 93.50 & 93.54 & \textbf{84.34} & \textbf{83.20} & 67.61 & 91.35 \\
        \rowcolor{pink!20}
        Data-free Unlearning & 93.87 & 72.41 & 71.82 & 77.64 & 83.23 & 75.39 & 71.58 & 67.59 & 66.84 & 80.67 & 84.12 & 83.89 & 88.93 & 70.67 & 69.14 & 68.59 & 76.69 \\
        \rowcolor{pink!20}
        Data-driven Unlearning & 98.73 & 98.93 & 99.23 & 99.55 & \textbf{98.90} & \textbf{99.44} & \textbf{99.32} & \textbf{99.30} & \textbf{99.33} & \textbf{99.02} & 99.14 & \textbf{99.29} & \textbf{98.87} & 76.83 & 74.63 & \textbf{73.65} & \textbf{94.50} \\
        \bottomrule
    \end{tabular}
    }
\end{table}

\begin{table*}[h]
\setlength{\tabcolsep}{3pt} 
\caption{AI-generated image detection performance on LSUN-BEDROOM.}
\label{compar_lsun}
\resizebox{\textwidth}{!}{%
\begin{tabular}{@{}lccccccccccccccccccc@{}}
\toprule
                     & \multicolumn{16}{c}{Models}                               & \multicolumn{2}{c}{} \\
 &
  \multicolumn{2}{c}{ADM} &
  \multicolumn{2}{c}{DDPM} &
  \multicolumn{2}{c}{iDDPM} &
  \multicolumn{2}{c}{Diffusion GAN} &
  \multicolumn{2}{c}{Projected GAN} &
  \multicolumn{2}{c}{StyleGAN} &
  \multicolumn{2}{c}{Unleashing Transformer} &
  \multicolumn{2}{c}{\multirow{-2}{*}{Average}} \\ \cmidrule(l){2-3} \cmidrule(l){4-5}\cmidrule(l){6-7}  \cmidrule(l){8-9}\cmidrule(l){10-11}\cmidrule(l){12-13}\cmidrule(l){14-15}
\multirow{-3}{*}{Methods}  &
  AUROC &
  AP &
  AUROC&
  AP &
  AUROC&
  AP &
  AUROC&
  AP &
  AUROC&
  AP &
  AUROC&
  AP &
  AUROC&
  AP &
  AUROC&
  AP &
  \\ \midrule
                     &&&&&&&\multicolumn{3}{c}{Training-free Methods}\\
 AEROBLADE &57.05 &58.37 &61.57 &61.49 &59.82 &61.06 &47.12 &48.25 &45.98 &46.15 &45.63 &47.06 &59.71 &57.34 &53.85 &54.25 \\
 \rowcolor{pink!20}
 Data-free Unlearning &\textbf{76.49}&\textbf{73.43}&\textbf{92.80}&\textbf{91.78} &\textbf{88.74} &\textbf{87.21} &\textbf{97.51} &\textbf{97.34} &\textbf{98.40} &\textbf{98.43} &\textbf{90.92} &\textbf{89.74} &\textbf{96.73} &\textbf{96.12} &\textbf{91.66}&\textbf{90.58}\\
 \midrule
 &&&&&&&\multicolumn{3}{c}{Training Methods}\\
CNNspot &64.83 &64.24 &79.04 &80.58 &76.95 &76.28 &88.45 &87.19 &90.80 &89.94 &95.17 &94.94 &93.42 &93.11 &84.09 &83.75\\
UnivFD &71.26 &70.95 &79.26 &78.27 &74.80 &73.46 &84.56 &82.91 &82.00 &78.42 &81.22 &78.08 &83.58 &83.48 &79.53 &77.94\\
DIRE  &57.19 &56.85 &61.91 &61.35 &59.82 &58.29 &53.18 &53.48 &55.35 &54.93 &57.66 &56.90 &67.92 &68.33  &59.00 &58.59\\
NPR &75.43 &72.60 &91.42 &90.89 &89.49 &88.25 &76.17 &74.19 &75.07 &74.59 &68.82 &63.53 &84.39 &83.67 &80.11 &78.25  \\
DRCT &74.59 &71.37 &85.45 &84.98 &87.17 &86.99 &94.19 &94.16 &95.96 &95.67 &93.92 &94.66 &89.51 &89.07 &88.68 &88.13 \\
 \rowcolor{pink!20}
 Data-driven Unlearning &\textbf{89.87} &\textbf{90.44} &\textbf{99.51} &\textbf{99.58} &\textbf{99.13} &\textbf{99.13} &\textbf{99.99} &\textbf{99.99} &\textbf{99.99} &\textbf{99.99} &\textbf{99.85} &\textbf{99.86} &\textbf{99.99} &\textbf{99.99} &\textbf{98.33} &\textbf{98.43} \\
 
 \bottomrule
\end{tabular}
}
\end{table*}

\begin{table*}[h]
\setlength{\tabcolsep}{3pt} 
\caption{AI-generated image detection performance (ACC, \%) on DiffusionForensics.}
\label{compar_DiffusionForensics}
\resizebox{\textwidth}{!}{%
\begin{tabular}{@{}lccccccccccccccccccccc@{}}
\toprule
                     & \multicolumn{16}{c}{Models}                               & \multicolumn{2}{c}{} \\
 &
  \multicolumn{2}{c}{ADM} &
  \multicolumn{2}{c}{DDPM} &
  \multicolumn{2}{c}{iDDPM} &
  \multicolumn{2}{c}{LDM} &
  \multicolumn{2}{c}{PNDM} &
  \multicolumn{2}{c}{VQ-Diffusion} &
  \multicolumn{2}{c}{SDV1} &
  \multicolumn{2}{c}{SDV2} &
  \multicolumn{2}{c}{\multirow{-2}{*}{Average}} \\ \cmidrule(l){2-3} \cmidrule(l){4-5}\cmidrule(l){6-7}  \cmidrule(l){8-9}\cmidrule(l){10-11}\cmidrule(l){12-13}\cmidrule(l){14-15}\cmidrule(l){16-17}
\multirow{-3}{*}{Methods}  &
  ACC &
  AP &
  ACC &
  AP &
  ACC &
  AP &
  ACC &
  AP &
  ACC &
  AP &
  ACC &
  AP &
  ACC &
  AP &
  ACC &
  AP &
   ACC &
  AP 
  \\ \midrule
CNNspot & 53.9 & 71.8 & 62.7 & 76.6 & 50.2 & 82.7 & 50.4 & 78.7 & 50.8 & 90.3 & 50.0 & 71.0 & 38.0 & 76.7 & 52.0 & 90.3 & 51.0 & 79.8 \\
UnivFD & 78.4 & 92.1 & 72.9 & 78.8 & 75.0 & 92.8 & 82.2 & 97.1 & 75.3 & 92.5 & 83.5 & 97.7 & 56.4 & 90.4 & 71.5 & 92.4 & 74.4 & 91.7 \\
Frank & 58.9 & 65.9 & 37.0 & 27.6 & 51.4 & 65.0 & 51.7 & 48.5 & 44.0 & 38.2 & 51.7 & 66.7 & 32.8 & 52.3 & 40.8 & 37.5 & 46.0 & 50.2 \\
Durall & 39.8 & 42.1 & 52.9 & 49.8 & 55.3 & 56.7 & 43.1 & 39.9 & 44.5 & 47.3 & 38.6 & 38.3 & 39.5 & 56.3 & 62.1 & 55.8 & 47.0 & 48.3 \\
SelfBland & 57.0 & 59.0 & 61.9 & 49.6 & 63.2 & 66.9 & 83.3 & 92.2 & 48.2 & 48.2 & 77.2 & 82.7 & 46.2 & 68.0 & 71.2 & 73.9 & 63.5 & 67.6 \\
GANDetection & 51.1 & 53.1 & 62.3 & 46.4 & 50.2 & 63.0 & 51.6 & 48.1 & 50.6 & 79.0 & 51.1 & 51.2 & 39.8 & 65.6 & 50.1 & 36.9 & 50.8 & 55.4 \\
Patchfor & 77.5 & \textbf{93.9} & 62.3 & 97.1 & 50.0 & 91.6 & \textbf{99.5} & \textbf{100.0} & 50.2 & 99.9 & \textbf{100.0} & \textbf{100.0} & 90.7 & \textbf{99.8} & \textbf{94.8} & \textbf{100.0} & 78.1 & 97.8 \\
\rowcolor{pink!20}
Data-free Unlearning & 79.8 & 85.0 & 87.5 & 94.3 & 88.3 & 95.3 & 80.2 & 89.5 & 94.2 & 98.7 & 92.3 & 98.0 & \textbf{93.1} & 97.9 & 92.8 & 97.9 & 88.5 & 94.6 \\
\rowcolor{pink!20}
Data-driven Unlearning & \textbf{86.9} & 93.4 & \textbf{98.5} & \textbf{99.9} & \textbf{98.4} & \textbf{99.9} & 95.4 & 99.0 & \textbf{99.1} & \textbf{100.0} & 99.2 & \textbf{100.0} & 92.2 & 97.5 & 93.5 & 98.1 & \textbf{95.4} & \textbf{98.5} \\
\bottomrule
\end{tabular}
}
\end{table*}



\subsection{The effect of structured pruning}
In our main experiment, we explore the effect of unstructured pruning, i.e., removing some of the weights in certain blocks individually. We further explore the effect of structured pruning, i.e., removing a certain block completely. As shown in Figure~\ref{struct_pruning}, removing shallow blocks usually gives stable results, whereas removing top blocks results in a more significant impact on the features of the natural image due to their closer connection to the output features. Thus, this results in poor detection performance. An exception is that when the second block is removed, the proposed method is completely unable to distinguish between natural images and AI-generated images. This may stem from the fact that the second block in DINOv2 is crucial for the extraction of the image features, and when the second block is removed, the model is unable to correctly extract the features of test images.

\begin{table}[h]
\centering
\caption{The effect of pruning strategy.}
\label{pruning_strategy}
\resizebox{0.55\textwidth}{!}{%
\begin{tabular}{c|p{45pt}|p{45pt}}
\toprule
Model & \hspace{6pt}AUROC & \hspace{6pt}AP \\
\midrule 
random pruning & \hspace{6pt}86.29 & \hspace{6pt}85.54 \\
pruning largest magnitude weights & \hspace{6pt}83.69 & \hspace{6pt}83.21 \\
pruning smallest magnitude weights & \hspace{6pt}92.20 & \hspace{6pt}91.45 \\
pruning and filling noise & \hspace{6pt}86.03 & \hspace{6pt}85.64 \\
\bottomrule
\end{tabular}%
}
\vspace{-0.2cm}
\end{table}

\begin{table*}[t]
\setlength{\tabcolsep}{3pt} 
\caption{\centering{Effectiveness of using feature similarity in the middle layer for detection.}}
\label{eff_middle_block}
\resizebox{\textwidth}{!}{%
\begin{tabular}{@{}lccccccccccccccccccccccc@{}}
\toprule
                     & \multicolumn{20}{c}{Generative Models}                               & \multicolumn{2}{c}{} \\
 &
  \multicolumn{2}{c}{ADM} &
  \multicolumn{2}{c}{ADMG} &
  \multicolumn{2}{c}{LDM} &
  \multicolumn{2}{c}{DiT} &
  \multicolumn{2}{c}{BigGAN} &
  \multicolumn{2}{c}{GigaGAN} &
  \multicolumn{2}{c}{StyleGAN XL} &
  \multicolumn{2}{c}{RQ-Transformer} &
  \multicolumn{2}{c}{Mask GIT} &
  \multicolumn{2}{c}{\multirow{-2}{*}{Average}} \\ \cmidrule(l){2-3} \cmidrule(l){4-5}\cmidrule(l){6-7}  \cmidrule(l){8-9}\cmidrule(l){10-11}\cmidrule(l){12-13}\cmidrule(l){14-15}\cmidrule(l){16-17}\cmidrule(l){18-19}
\multirow{-3}{*}{Block}  &
  AUROC &
  AP &
  AUROC&
  AP &
  AUROC&
  AP &
  AUROC&
  AP &
  AUROC&
  AP &
  AUROC&
  AP &
  AUROC&
  AP &
  AUROC&
  AP &
  AUROC&
  AP &
  AUROC&
  AP &\\ \midrule
 16 &70.32 &69.20 &65.32 &63.28 &82.62 &83.88 &85.69 &86.05 &88.20 &89.79 &76.42 &77.61 &82.59 &83.33 &81.48 &83.96 &86.64 &88.82 &79.92 &80.66 \\
 17 &63.93 &65.25 &61.22 &61.21 &72.37 &74.71 &80.19 &82.40 &84.65 &87.70 &61.02 &62.82  &66.01 &66.85 &73.34 &76.30 &80.59 &83.72 &71.48 &73.77\\
 18 &78.64 &76.76 &73.38 &70.61 &85.27 &84.43 &88.34 &88.05 &94.47 &94.37 &83.34 &83.75 &88.27 &88.82 &84.97 &84.02 &94.23 &94.51 &85.66 &85.04\\
 19 &86.90 &86.57 &81.04 &78.92 &88.03 &88.01 &89.80 &89.71  &96.51 &96.69 &90.86 &91.23 &93.55 &93.99 &89.77 &90.06 &96.54 &96.72 &90.33 &90.21\\
 20  &86.42 &86.31 &81.82 &79.37 &89.43 &89.60 &89.72 &89.36 &96.89 &96.99 &91.50 &91.57 &94.77 &94.96 &91.56 &91.74 &96.18 &96.30 &90.92 &90.69\\
21 &87.15 &87.57 &82.93 &80.84 &90.00 &90.33 &89.99 &90.36 &96.18 &96.42 &92.08 &92.55 &95.21 &95.52 &92.57 &93.28 &95.80 &96.08 &91.33 &91.44\\
  22 &91.04 &89.61 &85.53 &82.90 &88.00 &86.60 &86.37 &84.51 &96.29 &96.71 &94.51 &94.58 &96.65 &96.80 &94.90 &95.22 &95.97 &96.19 &92.14 &91.46 \\
 23  &91.97 &90.44 &86.82 &85.14 &87.62 &85.91 &85.74 &83.84 &96.37 &96.52 &94.39 &94.23 &96.47 &96.53 &95.19 &95.24 &95.27 &95.17 &92.20 &91.45 \\
 
 \bottomrule
\end{tabular}
}
\end{table*}

\begin{wrapfigure}{r}{0.5\textwidth}
  \centering
  \includegraphics[width=1\linewidth]{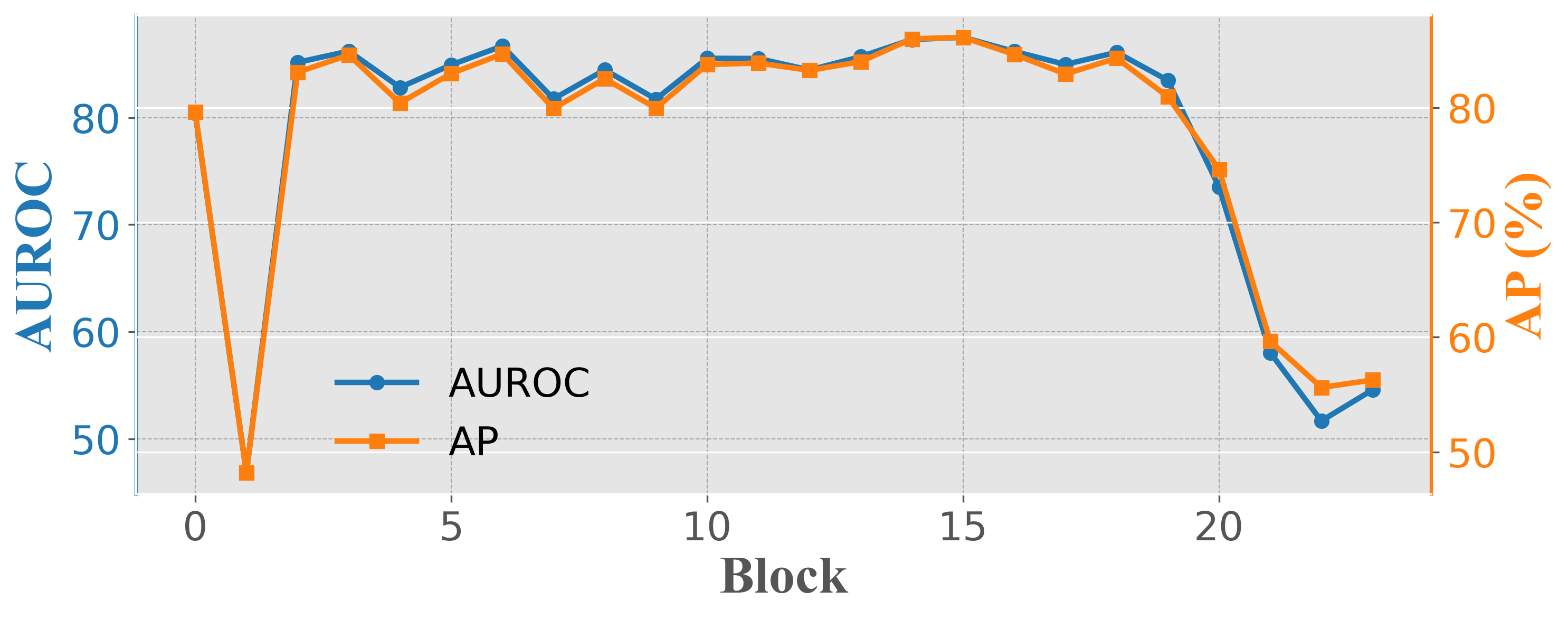}
  \caption{The effect of structured pruning. We obtain an unlearned model by completely dropping a full block.}
  \label{struct_pruning}
\end{wrapfigure}

\subsection{Using the similarity of the middle layer features as the decision score}
As shown in Table~\ref{eff_middle_block}, we further explore the effect of using the feature similarity of other middle layers as decision scores. Since we prune the weights of block 15, the outputs of the test samples on the learned and unlearned models differ from blocks 16 to block 23. Therefore, we explore the effect of feature similarity using the output of blocks from block 16 to block 23. The results show that using high-level features to compute the similarity could achieve good results. This is because there is a significant difference between the high-level features of the natural image and the generated image on the learned and unlearned models.


\subsection{Implementation Details}
\label{app:imple_detail}

For data-free unlearning, we leverage fully parameterized DINOv2 ViT-L/14 as the learned model. It has 24 transformer blocks, and we obtain a sparse model by pruning the parameters of $90\%$ of the minimum magnitude weights of the fc2 layer of its 16th transformer block, and use this model as the unlearned model. We use $1k$ natural images sampled from ImageNet and generated images generated by ProGAN to select hyperparameters. For data-driven unlearning, we leverage LoRA~\citep{DBLP:conf/iclr/HuSWALWWC22} for parameter-efficient fine-tuning. The LoRA layers are applied on the q\_proj and v\_proj layers of DINOv2. $lora\_r$ and $lora\_\alpha$ are set to 8. The margin $\gamma$ is set to 20. To optimize computational efficiency, we apply Low-Rank Adaptation (LoRA) exclusively to the 18th, 19th, and 20th blocks of the model and fine-tune for only three epochs. The model is optimized using the AdamW optimizer with a learning rate of $1 \times 10^{-5}$, $\beta_1 = 0.9$, $\beta_2 = 0.99$, and a weight decay of $0.01$. We report the average results under five different random seeds and report the standard deviation in Figure~\ref{eff_gamma}. Following CNNspot~\citep{DBLP:conf/cvpr/WangW0OE20}, data augmentation techniques including JPEG compression and Gaussian blur are employed to enhance robustness. For the IMAGENET, LSUN-BEDROOM, and DiffusionForensics benchmarks, the ProGAN dataset serves as the training set. For the GenImage benchmark, SDv1.4 dataset is used.  For the DRCT-2M benchmark, SDv2 dataset is used as training set.

When comparing classification accuracy with other methods, since our method is not a standard binary classifier, the traditional classification threshold of 0.5 is not applicable to our method. Consequently, we employed a validation set to determine an appropriate threshold. Specifically, this validation set consisted of 1,000 images generated by ProGAN and an equivalent number of natural images. We identified the threshold that maximized classification accuracy on this validation set as the optimal threshold for subsequent analyses. The determined optimal thresholds were 0.94287 for the data-free unlearning method and 0.90178 for the data-driven unlearning method, with classification accuracy calculated accordingly at these thresholds.

\subsection{Threshold sensitivity}

In our study, thresholds were selected using a validation set comprising natural images from ImageNet and generated images from ProGAN to compute classification accuracy. To rigorously evaluate threshold sensitivity, we conduct additional experiments by calibrating thresholds with various generated image sets and assessing accuracy on the GenImage dataset. The result is shown in Table~\ref{Threshold_sensitivity}, underscoring substantial variability in accuracy attributable to threshold selection, reinforcing the merit of threshold-free metrics for performance evaluation.

\subsection{Details of datasets}
\label{Details_of_Datasets}
\textbf{ImageNet and LSUN-BEDROOM.} The natural images and AI-generated images of ImageNet benchmark and LSUN-BEDROOM benchmark can be obtained from \url{https://github.com/layer6ai-labs/dgm-eval}, which are provided by \citep{DBLP:conf/nips/SteinCHSRVLCTL23}. The generated images of the ImageNet benchmark are generated with the following generative models: ADM, ADMG, BigGAN, DiT-XL-2, GigaGAN, LDM, StyleGAN-XL, RQ-Transformer, and Mask-GIT. The generated images of the LSUN-BEDROOM benchmark are generated with the following generative models: ADM, DDPM, iDDPM, StyleGAN, Diffusion-Projected GAN, Projected GAN, and Unleashing Transformers. 

\noindent \textbf{GenImage.} The natural images and AI-generated images can be obtained from \url{https://github.com/GenImage-Dataset/GenImage }. The images are provided by \citep{DBLP:conf/nips/ZhuCYHLLT0H023}. The generative model includes Midjourney, SD V1.4, SD V1.5, ADM, GLIDE, Wukong, VQDM, and BigGAN. The natural images come from ImageNet.

\noindent \textbf{Chameleon.} Chameleon is a a very challenging dataset and various detection methods perform unsatisfactorily on it, as all AI-generated images in this dataset have passed a human perception ”Turing Test”, i.e., human annotators have misclassified them as natural images. The images are provided by \citep{yan2024sanity}. The dataset can be obtained from \url{https://shilinyan99.github.io/AIDE/}. 

\noindent \textbf{DiffusionForensics.} The natural images and AI-generated images of DiffusionForensics can be obtained from \url{https://github.com/ZhendongWang6/DIRE}, which are provided by \citep{DBLP:conf/iccv/WangBZWHCL23}.  The generative model includes ADM, DDPM, iDDPM, LDM, PNDM, VQ-Diffusion, sdv1 and sdv2. 

\noindent \textbf{DRCT-2M.} The natural images of DRCT-2M come from CoCo and can be obtained from \url{https://cocodataset.org/#download}. AI-generated images of DRCT-2M can be obtained from \url{https://modelscope.cn/datasets/BokingChen/DRCT-2M/files}, which are provided by \citep{DBLP:conf/icml/ChenZYY24}.  The generative model includes  LDM, SDv1.4, SDv1.5, SDv2, SDXL, SDXL-Refiner, SD-Turbo, SDXL-Turbo, LCM-SDv1.5, LCM-SDXL, SDv1-Ctrl, SDv2-Ctrl, SDXL-Ctrl, SDv1-DR, SDv2-DR, SDXL-DR. 


\end{document}